%% file: root.tex
\documentclass[journal]{IEEEtran}  
\usepackage{amsmath} 
\usepackage{amsthm}
\theoremstyle{definition}

\newtheorem{assumption}{Assumption}
\usepackage{amssymb}  
\usepackage{bm}
\usepackage{graphicx}
\usepackage{algorithm}
\usepackage{algpseudocode}  
\usepackage{booktabs}
\makeatletter
\let\IEEEorigmakecaption\@makecaption
\makeatother
\usepackage[labelsep=period]{caption}
\usepackage{subcaption}

\usepackage{float}
\usepackage{diagbox}
\usepackage{multirow}
\usepackage{booktabs}
\usepackage{hyperref}
\usepackage{url}
\usepackage{amsthm}
\newtheorem{lemma}{Lemma}
\usepackage{svg}

\usepackage{xcolor} 
\usepackage{tabularx}
\usepackage[noprefix]{nomencl} 
\usepackage{pifont}
\makeatletter
\AtBeginDocument{%
  \let\@makecaption\IEEEorigmakecaption
  \setlength\abovecaptionskip{0.5\baselineskip}%
  \setlength\belowcaptionskip{0pt}%
}
\makeatother

\usepackage{cite}

\ifCLASSINFOpdf
\else
\fi
\begin{document}
%
\title{Simultaneous Contact Selection and Planning for Contact-Rich Manipulation with Cascaded Optimization}

\author{Zhe Zhang$^\dagger$, Xingrong Diao$^\dagger$, Haoxiang Liang, Han Yang, Bi-Ke Zhu, Dandan Zhang, Jiankun Wang, \textit{Senior Member}, \textit{IEEE} %
\thanks{$^\dagger$ denotes equal contribution.}}

\markboth{Journal of \LaTeX\ Class Files,~Vol.~18, No.~9, September~2020}%
{Simultaneous Contact Selection and Planning for Contact-Rich Manipulation with Cascaded Optimization}

\maketitle

\begin{abstract}
We propose an optimization-based framework for robust contact-rich manipulation. Recent contact-implicit methods enable online hybrid planning across contact modes, allowing closed-loop manipulation for a given target state and contact location sequence of the robot and object. However, most existing approaches lack the ability to autonomously reason and generate diverse contact location sequences and manipulation trajectories, i.e., active contact location selection, which limits their applicability to relatively simple tasks. Active contact location selection is challenging due to complementarity in contact dynamics and the sparse gradients, making the design of a unified framework for contact selection and planning difficult. To address these challenges, we introduce Simultaneous Contact Selection and Planning (SCSP), a cascaded optimization framework comprising Contact Selection Optimization (CSO) and Contact Planning Optimization (CPO). CSO leverages a surrogate contact model and discrete-continuous optimization to efficiently resolve the nonsmoothness and coupling in contact selection, enabling online global searching of optimal contact locations. CPO performs prior-guided contact planning by evaluating the reference contact locations produced by CSO and generating corresponding manipulation trajectories in real time for redundant manipulators. Extensive simulations and real-world experiments demonstrate that SCSP produces diverse manipulation behaviors and robust control under inaccurate dynamics and perceptual noise. We further validate the generalization of the framework on challenging manipulation tasks.

Project website: \href{https://sites.google.com/view/scsp-robot}{https://sites.google.com/view/scsp-robot}.
\end{abstract}


\section*{Nomenclature}
\addcontentsline{toc}{section}{Nomenclature}
\begin{IEEEdescription}[\IEEEusemathlabelsep\IEEEsetlabelwidth{$V_1,V_2,V_3$}]
\item[$p, \boldsymbol{p}$] A 3D point and its position vector. 
\item[$\theta, \boldsymbol{\psi}$] Angle and quaternion vector.
\item[$\mathcal{O}, \partial \mathcal{O}$] Object geometry and surface. 
\item[$\mathcal{R}, \partial \mathcal{R}$] Robot geometry and surface. 

\item[$h$] Time interval.
\item[$\boldsymbol{\lambda}, \boldsymbol{\beta}$] Contact force vector and contact impulse vector. 
\item[$\boldsymbol{\phi}$] Contact distance vector.
\item[$\boldsymbol{J}$] Contact Jacobian. 
\item[$\boldsymbol{n}, \boldsymbol{t}_i$] Normal vector and $i$-th tangental vector.
\item[{$\boldsymbol{x} = [ \boldsymbol{p}, \boldsymbol{\psi} ]$}] Pose vector consists of position and orientation.
\item[$\boldsymbol{q}, \dot{\boldsymbol{q}}, \ddot{\boldsymbol{q}}$] Robot configuration, velocity and acceleration.
\item[$\boldsymbol{s}$] System state consists of object pose $\boldsymbol{x}_o$ and  $\boldsymbol{q}$. 
\item[$\boldsymbol{v}$] Velocity vector in Cartesian space.
\end{IEEEdescription}
%

\input{introduction}

\input{preliminary}

\input{method}

\input{experiment}

\input{generalize}

\input{discussion}
\input{conclusion}

\bibliographystyle{IEEEtran}

\bibliography{reference}

\appendix

\input{appendix}

\end{document}

%% file: introduction.tex
\section{Introduction}
\label{sec:intro}

\IEEEPARstart{C}{ontact}-Rich manipulation requires robots to intentionally make and break contact with objects to drive them toward desired states. Specifically, the robot needs to plan the contact sequence consisting of approaching trajectory, contact locations, contact forces, and transitions between different contact modes (i.e., contact, separation, sliding, and sticking). Such contact planning is challenging due to the combinatorial complexity of contact modes and the inherent nonsmoothness arising from the hybrid nature of contact dynamics. Early works adopted hierarchical frameworks by separating continuous-state planning and discrete contact mode searching \cite{zito2012two, wu2020r3t, haustein2015kinodynamic, chen2021trajectotree, cheng2022contact}, or utilizing mixed-integer programs to handle the hybrid nature of contact \cite{aceituno2020global, marcucci2017approximate, hogan2020feedback}. However, this separation often suffers from poor scalability caused by the explosion of contact trajectories. Therefore, these methods offer limited manipulation strategies and poor robustness. 

Recent advances have been achieved through both learning-based methods and model-based methods. Learning-based methods \cite{kim2023pre, lianglearning, zhou2023learning, zhang2023learning, zhou2023hacman, cho2024corn, jiang2024hacman++, le2025enhancing, lyu2025dywa} can capture the nonsmooth nature of contact dynamics through offline reinforcement learning (RL), thus efficiently handle complex contact-rich manipulation tasks that are difficult for model-based methods. However, they are often limited to generalization across
diverse object types \cite{kim2023pre, lianglearning, zhou2023learning} or new task settings \cite{zhou2023hacman, jiang2024hacman++, le2025enhancing, cho2024corn}. This constitutes a major bottleneck in developing foundation models for manipulation. DyWA \cite{lyu2025dywa} proposes a world-action model architecture and shows impressive results in 6D non-prehensile manipulation with generalization across diverse object geometries. However, a generalizable framework remains to be explored.

\begin{figure}[!t]
 \centerline{\includegraphics[width=\columnwidth]{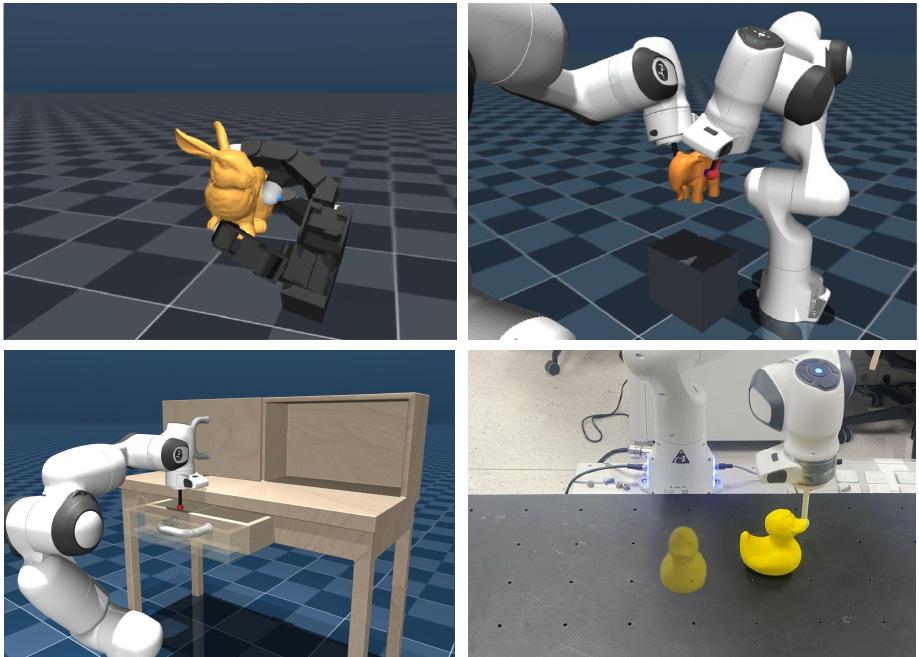}}
\caption{SCSP enables autonomous reasoning for contact sequence selection and online generation of diverse manipulation trajectories, allowing it to be applied to a variety of challenging contact-rich manipulation tasks.}
\label{fig:scsp_fig1}
\end{figure}
The representative model-based method is contact-implicit trajectory optimization (CITO), which was first applied to locomotion \cite{posa2014direct, wensing2023optimization, li2025surprising} and later extended to manipulation. We mainly focus on Contact-Implicit Model Predictive Control (CIMPC), which solves the contact planning problem online. These methods leverage time-stepping contact dynamics \cite{stewart2000implicit, anitescu2006optimization} and incorporate a convexified contact model within trajectory optimization, enabling hybrid contact planning in a unified framework. CIMPC can compute control inputs autonomously based solely on the object's current pose, goal pose, and dynamics information to establish and break contact with the object. Existing methods primarily focus on investigating different approximation formulations and solution strategies. For instance, C3 \cite{aydinoglu2024consensus} locally linearizes the LCS model and solves it using various solvers. CQDC \cite{pang2023global} proposed a convexified contact model significantly reducing computational complexity, and subsequent works \cite{suh2025dexterous}\cite{jiang2025robust} have built upon this idea to design even more efficient algorithms. Furthermore, CF-MPC \cite{jin2024complementarity} derives a complementarity-free model to improve computational efficiency. Although such online hybrid planning has increased manipulation diversity and enabled these methods to accomplish more complex tasks, the lack of active contact selection remains a key bottleneck to their performance. These methods are largely limited to manipulating objects with simple geometries or operating under accurate dynamics and perceptions. Additionally, CIMPC is prone to converging to local optima when the robot's initial guess is poor.

Active contact selection is crucial in complex contact-rich manipulation, as it enables the generation of sufficiently diverse manipulation strategies to avoid local optima. However, most existing model-based methods lack this capability. Current approaches either rely on hand-crafted cost functions to guide contact or utilize a given reference trajectory as a contact prior for planning. This constrains their manipulation trajectories to a narrow distribution, making them prone to failure when facing out-of-distribution (OOD) tasks. \cite{suh2025dexterous, jiang2025robust} rely on high-quality prior knowledge of contact locations and are unable to perform a global search for the optimal contact location. \cite{li2025surprising} developed a novel optimization formulation named CRISP to achieve global planning of contact force and contact location, but it depends on explicit geometric representations of objects and is limited to 2D pushing tasks involving simple geometries. IDTO \cite{kurtz2026inverse} removes the shooting phase from CIMPC and reformulates the problem as a trajectory optimization and inverse dynamics problem, enabling real-time computation on highly redundant robots. However, this method is limited to simple regular objects to ensure gradient continuity during optimization. STOCS-3D \cite{zhang2025simultaneous} enables contact planning for arbitrarily complex geometries by formulating CITO as a semi-infinite problem, but it is restricted to offline computation and has not been validated on robot manipulation tasks. \cite{venkatesh2025approximating, bui2025push} achieve 6D nonprehensile manipulation by sampling points across the object’s surface and evaluating parallel CIMPC instances to determine the optimal contact location, but they suffer from prohibitive computational costs. \cite{pan2025spider} uses a reference trajectory as prior guidance for contact selection, but it is unable to achieve autonomous reasoning. \cite{xie2026touch} explores the use of an RL policy to generate the optimal contact location and contact planning parameters, enabling online contact selection and planning. This work inspiring for our approach, but it still lacks validation of generalization across different topologies. Designing a unified planning framework for simultaneous contact selection and contact planning is challenging. The complementarity inherent in contact dynamics introduces nonsmoothness into contact selection. Moreover, since the choice of contact locations depends on the object's state evolution at the time of contact, hybrid planning across contact-free and contact-rich stages renders contact selection a sparse-gradient optimization problem. As a result, a generative framework for unified contact selection and planning remains absent.

\begin{table*}[ht]
\centering
\caption{Comparison of Contact Planning Methods.}
\label{tab:comparison}
\small
\begin{tabularx}{\textwidth}{l c c c c c c}
\toprule
\textbf{Methods} & \textbf{Perception} & \textbf{Contact Selection} & \textbf{6D Manip. Exp.} & \textbf{Real-Time} & \textbf{Arbitrary Obj. Shape} & \textbf{Generalization} \\
\midrule
C3~\cite{aydinoglu2024consensus}        & Exact & \ding{55}  & \ding{55} & \ding{51} & \ding{55} & \ding{51} \\
CRISP~\cite{li2025surprising} & Exact & \ding{51} & \ding{55} & \ding{51} & \ding{55} & \ding{55} \\
IDTO~\cite{kurtz2026inverse}        & Exact & \ding{55}  & \ding{51} & \ding{51} & \ding{55} & \ding{51} \\
CQDCs~\cite{pang2023global}  & Exact & \ding{55}  & \ding{51} & \ding{51} & \ding{55} & \ding{51} \\
CF-MPC~\cite{jin2024complementarity} & Exact & \ding{55} & \ding{51} & \ding{51} & \ding{51} & \ding{51} \\
STOCS-3D~\cite{zhang2025simultaneous} & Exact & \ding{51}  & \ding{51} & \ding{55} & \ding{51} & \ding{51} \\
DyWA~\cite{lyu2025dywa}  & Vision & \ding{51}  & \ding{51} & \ding{51} & \ding{51} & \ding{55} \\
Sampling~\cite{venkatesh2025approximating}  & Vision & \ding{51}  & \ding{51} & \ding{51} & \ding{51} & Not validated \\
\midrule
\textbf{SCSP} & Vision & \ding{51} & \ding{51} & \ding{51} & \ding{51} & \ding{51} \\
\bottomrule
\end{tabularx}
\end{table*}
In this work, we propose a novel cascaded optimization framework, Simultaneous Contact Selection and Planning (SCSP), for complex contact-rich manipulation. SCSP achieves superior dexterity and diversity in manipulation strategies compared to traditional model-based methods, while maintaining robustness under inaccurate model and perception noise. Our key insight is that contact-rich manipulation naturally decomposes into two stages: Contact Selection Optimization (CSO), which globally searches for optimal contact locations by fast simulating state transitions under contact forces, and Contact Planning Optimization (CPO), which performs hybrid planning of manipulation trajectories, including approaching, contact, and contact-switching. In CSO, we analyze why existing CIMPC formulation struggle with active contact selection. The challenge stems from the nonsmoothness caused by complementarity in contact dynamics and the inherent sparsity of gradients in the optimization landscape. To overcome these challenges and enable online selection of contact locations, we discretize the problem by sampling candidate points on the object surface, reformulating CSO as a discrete-continuous optimization that resolves the nonsmoothness. Furthermore, we derive a Surrogate Contact Model (SCM) to approximate the original contact dynamics, alleviating complex coupling and complementarity while significantly improving computational efficiency. This reduces the contact selection problem to an efficient Mixed-Integer Quadratic Programming (MIQP) formulation, where each single-step QP can be solved at sub-millisecond timescales. In CPO, we introduce a prior-guided contact planning formulation. First, we employ a ranking strategy to comprehensively evaluate and selectively execute the reference contact locations predicted by CSO. Then, we design a simple yet effective objective that unifies hybrid planning of contact location switching, approach, and contact within a single optimization. Finally, we propose two efficient solvers for CPO, which are applicable to robots with diverse kinematic configurations.

Comprehensive comparison with celebrated contact-rich manipulation methods is provided in Table \ref{tab:comparison}. SCSP is a novel framework that does not rely on exact perception, enables global contact selection, has been validated in 6D manipulation experiments, supports real-time planning, can manipulate objects with arbitrary geometries, and generalizes to different tasks as shown in Fig. \ref{fig:scsp_fig1}.

The main contributions are:
\begin{itemize}
    \item A novel cascaded optimization framework for contact-rich manipulation is presented, comprising a Contact Selection Optimization (CSO) module and a Contact Planning Optimization (CPO) module. The framework addresses the challenges of integrating active contact selection into CIMPC methods, enabling greater diversity in manipulation strategies and improved robustness under inaccurate dynamics and perceptual noise. Moreover, the method can be generalized to the manipulation of arbitrarily complex object geometries, a variety of robot configurations, and diverse contact-rich tasks.

    \item An efficient contact selection algorithm is proposed. We transform the nonsmooth contact selection optimization into a discrete–continuous problem using a mesh sampling and exchange method. A surrogate contact model is then employed to simplify the contact dynamics, reducing the original problem to a tractable MIQP and enabling fast global search for the optimal contact locations.

    \item A novel prior-guided contact planning formulation is presented. A ranking strategy is employed to evaluate the quality of contact priors, establishing a more expressive connection between prior guidance and contact planning. Additionally, an effective objective is designed for online hybrid planning of the manipulation trajectory, covering approaching, contacting, and switching between contacts.

    \item Extensive simulations and real-world experiments are conducted to validate the superiority and robustness of our approach over existing methods. To the best of our knowledge, SCSP is the first model-based method to demonstrate real-world 6D nonprehensile manipulation of arbitrarily complex object geometries using only vision perception, without the assistance of markers or motion capture systems.
\end{itemize}

The proposed SCSP overcomes the theoretical limitations of existing contact-implicit methods. Although recent RL approaches have achieved success in many locomotion and manipulation tasks, our goal is to make diverse and robust contact planning more interpretable. Furthermore, inspired by the structure of SCSP, we aim to investigate the role of world models and their guidance in contact-rich planning in a principled manner.

The rest of the paper is organized as follows. Section \ref{sec:preliminaries} reviews the preliminaries. The CSO and CPO are introduced in Sections \ref{sec:cso} and \ref{sec:cpo}, respectively. Section \ref{sec:experiment} presents the experimental results, and section \ref{sec:generalize} demonstrates the generalization of the framework on other contact-rich manipulation tasks. Finally, discussion and limitations are presented in Section \ref{sec:discussion}, and conclusions are drawn in Section \ref{sec:conclusion}.

%% file: preliminary.tex
\section{Preliminaries}
\label{sec:preliminaries}

We begin with a brief overview to help readers build the preliminary background and better understand the remainder of this article. Section \ref{preliminaries:contact_dynamics} introduces time-stepping contact dynamics and the associated complementarity problem. Section \ref{preliminaries:contact_planning} presents the mathematical definition of contact planning, and describes a representative solution, namely CIMPC. Section \ref{preliminaries:complementarity_free} reviews the recently proposed complementarity-free model, which significantly improves computational efficiency and provides inspiration for our work.

\subsection{Contact Dynamics}\label{preliminaries:contact_dynamics}

Contact dynamics provides a physical description of interactions among the robot, object, and environment. 
Here, we consider the commonly used Anitescu model \cite{anitescu2006optimization}, which is formulated as the following time-stepping equation: 
\begin{subequations}\label{equ.quasi_dyn}
\begin{gather}
\epsilon\boldsymbol{M}_o\boldsymbol{v}_o = 
    h\boldsymbol{\tau}_o + \sum\nolimits_{i=1}^{n_c}{\boldsymbol{J}}_{o,i}^{\top}\boldsymbol{\beta}_i,
    \label{equ.quasi_dyn_1} \\
    h\boldsymbol{K}_r(h\boldsymbol{v}_r - \boldsymbol{u}) = 
    h\boldsymbol{\tau}_r + \sum\nolimits_{i=1}^{n_c}{\boldsymbol{J}}_{r,i}^{\top}\boldsymbol{\beta}_i.
    \label{equ.quasi_dyn_2}
\end{gather}
\end{subequations}
where we typically simplify $\boldsymbol{\beta}_i$ as  $\boldsymbol{\beta}_i \approx h \boldsymbol{\lambda}_i$ within a short time interval $h$. $\epsilon\boldsymbol{M}_o$, $\boldsymbol{v}_o$, $\boldsymbol{\tau}_o$ and $\boldsymbol{J}_{o,i}$ are the regularized mass matrix, velocity change, non-contact force and contact jacobian of the object, $\boldsymbol{K}_r$, $\boldsymbol{v}_r$, $\boldsymbol{\tau}_r$ and $\boldsymbol{J}_{r,i}$ are the robot's stiffness matrix, velocity, non-contact torque and contact jacobian. $\boldsymbol{u}$ is the control inputs and is defined as $\boldsymbol{u}=\Delta \boldsymbol{q}$.

The above equations describe, from the perspective of momentum balance, how the robot influences the object’s state through contact, as well as how the object interacts with the environment. Here, we follow \cite{pang2021convex, jin2024complementarity} and consider the robot to be under impedance control \cite{hogan1984impedance}, modeling the robot–object interaction as a spring system. The contact constraints between contact forces (or contact impulses) and system motion are typically formulated as a cone complementarity formulation \cite{anitescu2010iterative} as follows:
\begin{equation}\label{equ.complementarity_constraint}
    \\ \mathcal{K}_i\ni \boldsymbol{\lambda}_i \perp \boldsymbol{J}_i\boldsymbol{v}+\frac{1}{h}\begin{bmatrix}
        {\phi_i}\\
        0\\
        0
    \end{bmatrix}\in \mathcal{K}_i^*, \quad 
    \forall i=\{1,...,n_c\},
\end{equation}
where $\mathcal{K}_i = \left\{ \boldsymbol{\lambda}_i \in \mathbb{R}^3 \,\middle|\, 
\mu_i \lambda^n_i \ge \left\| \lambda^d_i \right\| \right\}$ is the Coulomb frictional cone, and $\mathcal{K}_i^*$ is the dual cone as $\boldsymbol{v} \in 
\left\{ \boldsymbol{v} \,\middle|\, 
\boldsymbol{J}^n_i \boldsymbol{v} + \frac{\phi_i}{h} 
\ge 
\mu_i \left\| \boldsymbol{J}^d_i \boldsymbol{v} \right\| 
\right\}$.

Equations \eqref{equ.quasi_dyn} and \eqref{equ.complementarity_constraint} are proved to be the Karush-Kuhn-Tucker (KKT) \cite{avriel2003nonlinear} optimality conditions for the following primal optimization \cite{anitescu2006optimization}:
\begin{subequations}\label{equ.contact_model}
\begin{align}
\min_{\boldsymbol{v}} \quad 
& \frac{1}{2}h^2\boldsymbol{v}^\top \boldsymbol{Q}\boldsymbol{v}
- h \boldsymbol{v}^\top \boldsymbol{b}(\boldsymbol{u})
\label{equ.qs_model_obj} \\
\text{s.t.} \quad 
& \boldsymbol{J}_i\boldsymbol{v}
+ \frac{1}{h}
\begin{bmatrix}
\phi_i \\
0 \\
0
\end{bmatrix}
\in \mathcal{K}_i^*, 
\quad i \in \{1,\dots,n_c\},
\label{equ.contact_model_con}
\end{align}
\end{subequations}
where $\boldsymbol{Q} \in \mathbb{R}^{(n_o+n_{\text{DoF}}) \times (n_o+n_{\text{DoF}})}$ and $\boldsymbol{b}(\boldsymbol{u}) \in \mathbb{R}^{n_o+n_r}$ are 
\begin{equation}\label{equ.q_matrix}
\boldsymbol{Q}:=\begin{bmatrix}
    \epsilon \boldsymbol{M}_o/h^2 & \boldsymbol{0}\\
    \boldsymbol{0}& \boldsymbol{K}_r
\end{bmatrix},
\quad 
    \boldsymbol{b}(\boldsymbol{u}):=\begin{bmatrix}
        \boldsymbol{\tau}_o\\
        \boldsymbol{K}_r \boldsymbol{u}+\boldsymbol{\tau}_r
    \end{bmatrix},
\end{equation}
 respectively. Here $n_o, n_{\text{DoF}}$ are the dimensions of the object and robot states, and $n_c$ is the number of contacts. (\ref{equ.contact_model}) is a second-order cone program (SOCP) and is often approximated with a QP by reformulating (\ref{equ.contact_model_con}) with polyhedral cone constraint:
\begin{equation}
    (\boldsymbol{J}^n_i-\mu_i \boldsymbol{J}^d_{i,j}) \boldsymbol{v} + \frac{\phi_i}{h} \ge 0, i \in \left \{ 1,\dots n_c \right \}, j \in \left \{ 1,\dots n_d \right \}, 
\end{equation}
where $n_d$ is the number of friction cone facets. Solving \eqref{equ.contact_model} is equivalent to performing one step of physically consistent contact simulation. The solution $\boldsymbol{v}$ can then be used to integrate the current state $\boldsymbol{q}_t$ by $\boldsymbol{q}_{t+1}=\boldsymbol{q}_t\oplus h \boldsymbol{v}$, thereby performing a one-step contact simulation. \eqref{equ.contact_model} and its variants are widely used in physics simulators.

\subsection{Model-based Contact Planning}\label{preliminaries:contact_planning}

Contact planning refers to generating trajectories for approaching, making contact, and breaking contact with the object to manipulate it to a target state. Model-based contact planning methods integrate contact dynamics and solve trajectory optimization problems. As a representative contact planning method, CIMPC is formulated as the following optimization problem:
\begin{equation}
    \begin{aligned}
\min_{\boldsymbol{u}_{0:N-1}} \quad & 
 {\textstyle \sum_{i=0}^{N-1}}  \ell(\boldsymbol{x}_k, \boldsymbol{u}_k) + V_f(\boldsymbol{x}_N) \\
\text{s.t.}\quad & \boldsymbol{x}_{k+1} = f(\boldsymbol{x}_k, \boldsymbol{u}_k), \quad k=0,\dots,N-1, \\
& \boldsymbol{x}_k \in \mathcal{X}, \quad k=0,\dots,N, \\
& \boldsymbol{u}_k \in \mathcal{U}, \quad k=0,\dots,N-1.
\end{aligned}\label{eq:CIMPC_opt}
\end{equation}
in which $\boldsymbol{x}_k$ is the system state, including object pose $\boldsymbol{q}_k$ and robot configuration, $\boldsymbol{u} = \Delta \boldsymbol{q}$ is the displacement, and $f$ is the contact dynamics. The stage cost $\ell$ can take various forms and is typically used to guide the process of approaching the contact location. The terminal cost $V_f$ is usually defined as the error between the object pose and the desired pose:
\begin{equation}\label{equ:obj_pose_cost}
\ell_{\text{pose}}
= w_{\text{pos}} \left\| \boldsymbol{p}_{\text{obj}} - \boldsymbol{p}_{\text{goal}} \right\|_2^2
+ w_{\text{quat}} \left( 1 - \left( \boldsymbol{\psi}_{\text{obj}}^\top \boldsymbol{\psi}_{\text{goal}} \right)^2 \right).
\end{equation}

The hybrid nature of contact dynamics renders \eqref{eq:CIMPC_opt} a piecewise-smooth nonlinear complementarity problem (NCP), which existing optimization methods are unable to solve online. Although \eqref{equ.contact_model} simplifies the contact dynamics, directly using \eqref{equ.contact_model} as the model $f$ remains computationally too expensive for online planning. In contrast, CIMPC achieves a local approximation of the original contact dynamics via Taylor expansion or linearization. Although this sacrifices accuracy, it yields smooth contact models that are compatible with existing efficient solvers and enable real-time optimization \cite{pang2023global, suh2025dexterous,jiang2025robust}.

\subsection{Complementarity-free Model}\label{preliminaries:complementarity_free}

In the following, we introduce an efficient contact model proposed in \cite{jin2024complementarity}, which has provided significant inspiration for our approach. (\ref{equ.contact_model}) has the dual problem:
\begin{multline}\label{equ.dual_prob_relax}
        \max_{\boldsymbol{\beta}\geq \boldsymbol{0}}\quad  -\frac{1}{2h^2} \boldsymbol{\beta}^\top \left(
    \boldsymbol{\tilde{J}} \boldsymbol{Q}^{-1} \boldsymbol{\tilde{J}}^\top+\boldsymbol{R}
    \right) \boldsymbol{\beta}\\-\frac{1}{h}(\boldsymbol{\tilde{J}} \boldsymbol{Q}^{{-}1}\boldsymbol{b}+\boldsymbol{\tilde{\phi}})^\top\boldsymbol{\beta}-\frac{1}{2} \boldsymbol{b}^\top\boldsymbol{Q}^{{-}1}\boldsymbol{b}.
\end{multline} 
in which $\boldsymbol{R}$ is small regularization term which is commonly used to enhance numerical stability, and:
\begin{equation}
    \boldsymbol{\tilde{J}}=\begin{bmatrix}\boldsymbol{J}_1^n-\mu_1\boldsymbol{J}_{1,1}^t
 \\ \dots 
 \\ \boldsymbol{J}_1^n-\mu_1\boldsymbol{J}_{1,n_t}^t
 \\ \vdots 
 \\ \boldsymbol{J}_{n_c}^n-\mu_{n_c}\boldsymbol{J}_{{n_c},1}^t
 \\ \dots 
 \\ \boldsymbol{J}_{n_c}^n-\mu_{n_c}\boldsymbol{J}_{{n_c},n_t}^t

\end{bmatrix}, \boldsymbol{\tilde{\phi}}=\begin{bmatrix}\phi_1
 \\ \dots
 \\ \phi_1
 \\ \vdots 
 \\ \phi_{n_c}
 \\ \dots
 \\ \phi_{n_c}
\end{bmatrix}, \boldsymbol{\beta}=\begin{bmatrix} \beta_{1,1}
 \\ \dots 
 \\ \beta_{1,n_d}
 \\ \vdots 
 \\ \beta_{n_c,1}
 \\ \dots 
 \\ \beta_{n_c,n_d}
\end{bmatrix}.
\end{equation}
Here, $\boldsymbol{\beta}$ is not only a dual variable but also has a physical interpretation as the contact impulse. Specifically, the solution of (\ref{equ.dual_prob_relax}) satisfies the following complementarity constraints:  
\begin{equation}\label{equ.dual_prob_relax_sol}
        \boldsymbol{0}\leq  \boldsymbol{\beta} \perp
        \frac{1}{h} \left(
    \boldsymbol{\tilde{J}} \boldsymbol{Q}^{-1} \boldsymbol{\tilde{J}}^\top{+}\boldsymbol{R}
    \right)\boldsymbol{\beta} 
    +\left(\boldsymbol{\tilde{J}} \boldsymbol{Q}^{{-}1}\boldsymbol{b}+{\boldsymbol{\tilde{\phi}}}\right)\geq \boldsymbol{0}.
\end{equation}
Assuming there exists a positive-definite diagonal matrix that satisfies
$\boldsymbol{K}(\boldsymbol{q}) = (\boldsymbol{\tilde{J}} \boldsymbol{Q}^{-1} \boldsymbol{\tilde{J}}^\top{+}\boldsymbol{R})^{-1}$,
we get the closed-form solution of (\ref{equ.dual_prob_relax}) as:
\begin{equation}
    \boldsymbol{\beta}^+= \max\left(
    -h\boldsymbol{K}(\boldsymbol{q})\big(\boldsymbol{\tilde{J}} \boldsymbol{Q}^{{-}1}\boldsymbol{b}+{\boldsymbol{\tilde{\phi}}}\big), \boldsymbol{0}
    \right).
\end{equation}
So the complementary-free model is written as:
\begin{equation}\label{equ:cf_model}
    f_{\text{cf}}:\begin{cases}
    \boldsymbol{v}^+ ={\frac{1}{h}\boldsymbol{Q}^{-1}\boldsymbol{b}}+ 
\frac{1}{h}\boldsymbol{Q}^{-1}\boldsymbol{\tilde{J}}^\top \boldsymbol{\beta}^+,
     \\ 
    \boldsymbol{x}_{k+1}=\boldsymbol{x}_{k}+h(\boldsymbol{v}_{k}+\boldsymbol{v}^+).
\end{cases}
\end{equation}
where $\boldsymbol{\beta}$ is the contact impulse of all the contact points and $\boldsymbol{v}^+$ is the combination of object velocity $\boldsymbol{v}_o$ and robot velocity $\boldsymbol{v}_r$. As discussed in \cite{jin2024complementarity}, this assumption can be viewed as a relaxation of the cone constraint, which is experimentally validated to be acceptable and significantly improves computational efficiency.

%% file: method.tex
\begin{figure*}
 \centerline{\includegraphics[width=\textwidth]{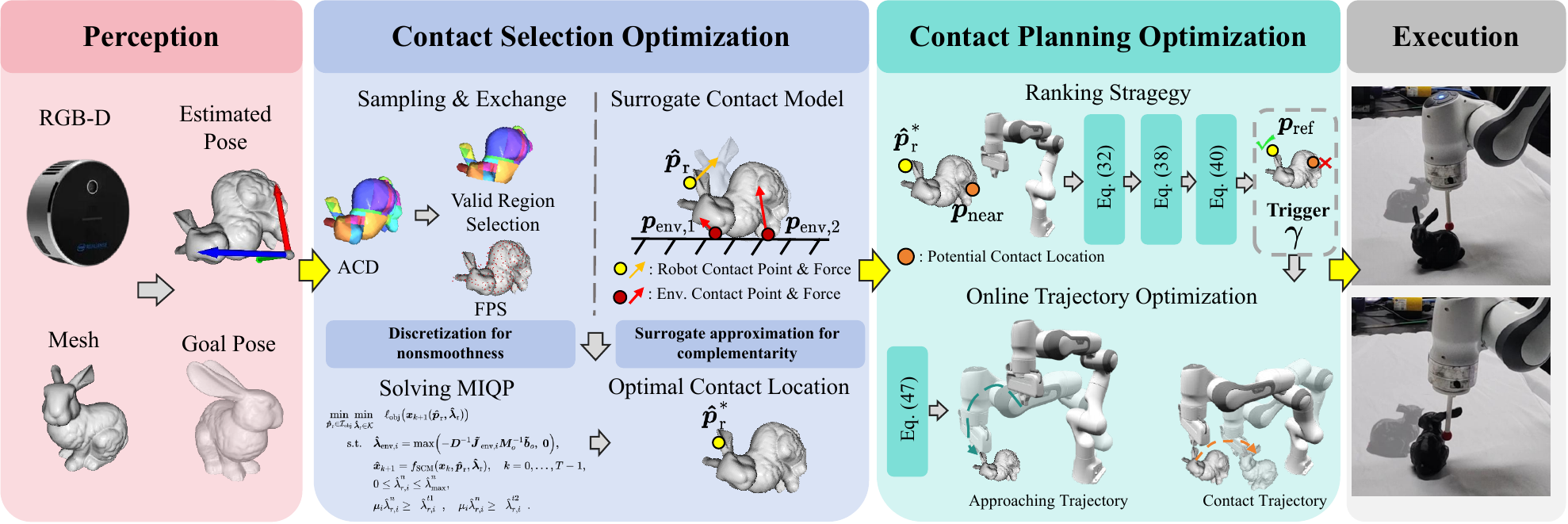}}
\caption{Framework structure of SCSP. SCSP consists of Contact Selection Optimization (CSO) and Contact Planning Optimization (CPO). CSO discretizes nonsmooth geometry through the Sampling \& Exchange method and introduces a surrogate contact model for efficient approximation of contact dynamics. By solving an MIQP, CSO online computes the globally optimal contact location $\boldsymbol{\hat{p}}_{\text{r}}^*$ as contact prior. The ranking strategy in CPO further evaluates whether the nearest point $\boldsymbol{p}_{\text{near}}$ can serve as a feasible suboptimal substitute for $\boldsymbol{\hat{p}}_{\text{r}}^*$ as the reference contact location, after which trajectory optimization is solved online to generate the approach and contact trajectories.}
\label{fig:framework}
\end{figure*}

\section{Contact Selection Optimization}
\label{sec:cso}

In this section, we present the detailed design of CSO. The CSO allows a global search for optimal contact locations on the object, which is crucial for contact-rich manipulation and enhances the capability to perform complex tasks.   
We begin by analyzing the key challenges of integrating contact selection into CIMPC, which limit the diversity of existing methods. We then demonstrate the main difficulties in contact selection optimization. Finally, we present the detailed design of our CSO module for online contact selection.

\subsection{Why Contact Selection is necessary?} 

Contact selection refers to choosing the contact locations $\left \{ \boldsymbol{p}_{\mathcal{O}, i} \right \} \in \partial \mathcal{O}$ and $\left \{ \boldsymbol{p}_{\mathcal{R}, i} \right \} \in \partial \mathcal{R}$ according to a given task specification. Since $\partial \mathcal{O}$ is a low-dimensional manifold embedded in Cartesian space, the set of robot–object contact states (i.e., the contact-rich stage) also lies on a low-dimensional manifold $\mathcal{S}_{\text{contact}}$ within the full state space $\mathcal{S}$. Task-related costs are defined only on $\mathcal{S}_{\text{contact}}$, whereas the remaining state space $\mathcal{S}_{\text{free}}$ is gradient-sparse. Moreover, CIMPC methods typically employ a short\footnote{For computational efficiency, these methods usually adopt a local approximation of the contact dynamics, which limits the valid region (i.e., the trust region) of the model and therefore necessitates a short planning horizon.} and fixed planning horizon with a fixed step size, which prevents reliable exploration of cost differences arising from different contact locations when starting from an arbitrary contact-free state $\boldsymbol{s}_k \in \mathcal{S}_{\text{free}}$. As a result, integrating contact selection into CIMPC, which must hybridly plan both contact-free and contact-rich trajectories, is challenging. Existing CIMPC methods often rely on manually designed costs or reference trajectories to guide contacts. Such designs effectively fix the contact locations and constrain the resulting manipulation trajectories to a narrow distribution, limiting the diversity of achievable behaviors. Moreover, the contact location is not a direct control variable in robot-centric CIMPC but is determined implicitly through nonsmooth collision checking, which further introduces complexity.

Rather than integrating contact selection into the CIMPC formulation, we study how to formulate contact selection as a separate optimization problem. As illustrated in Fig. \ref{fig:framework}, we first simplify the problem by ignoring the planning of the robot contact locations $\left \{ \boldsymbol{p}_{\mathcal{R}, i} \right \}$ and thereby obtain the following object-centric optimization problem:
\begin{equation}
    \begin{aligned}
        \min_{\boldsymbol{p}_{\text{r}} \in \partial \mathcal{O}, \boldsymbol{\lambda} \in \mathcal{K}} & 
        \ell_{\text{cso}}(\boldsymbol{x}_{k+1}) \\
        \text{s.t.} \quad 
        & \boldsymbol{x}_{k+1} = f(\boldsymbol{x}_k, \boldsymbol{p}_{\text{r}}, \boldsymbol{\lambda}),
        \quad k = 0, \dots, T-1.
    \end{aligned}\label{scsp}
\end{equation}
where $\mathcal{K}$ is the set of feasible contact forces and $\ell_{\text{cso}}(\cdot)$ is the objective of contact selection. $\boldsymbol{p}_{\text{r}} \in \mathbb{R}^{3 \times n_{\text{r}}}$ represents the locations where the robot applies contacts to the object, with $n_{\text{r}}$ denotes the number of robot contacts. \eqref{scsp} addresses the gradient sparsity issue arising from hybrid planning in the original contact-rich manipulation problem by defining CSO to solve for optimal contact locations and contact forces on $\mathcal{S}_{\text{contact}}$. The neglected process of the robot approaching $\mathcal{S}_{\text{contact}}$ from its current state is the primary source of the gradient sparsity. Note that \eqref{scsp} is hard to solve due to the coupling of contact location $\boldsymbol{p}_{\text{r}}$ and the contact force $\boldsymbol{\lambda}=\left [ \boldsymbol{\lambda}_{\text{env}}, \boldsymbol{\lambda}_{\text{r}} \right ]$. Therefore, we adopt a commonly used simplification for handling nonconvex optimization problems to reformulate (\ref{scsp}) as subproblems:
\begin{equation}
    \min_{\boldsymbol{p}_{\text{r}} \in \partial \mathcal{O}}
    \quad V(\boldsymbol{p}_{\text{r}})\label{equ:outer_loop},
\end{equation}
in which:
\begin{equation}
\begin{aligned}
V(\boldsymbol{p}_{\text{r}})
:= \min_{\boldsymbol{\lambda}_{\text{r}} \in \mathcal{K}} \quad &
\ell_{\text{cso}}(\boldsymbol{x}_{k+1}) \\
\text{s.t.} \quad
& \boldsymbol{x}_{k+1}
= f(\boldsymbol{x}_k, \boldsymbol{p}_{\text{r}}, \boldsymbol{\lambda}_{\text{r}}, \boldsymbol{\lambda}_{\text{env}}),
\quad k = 0, \dots, T-1.
\end{aligned}\label{equ:inner_loop}
\end{equation}
The physical interpretation of \eqref{equ:outer_loop} and \eqref{equ:inner_loop} is that a contact location is regarded as optimal if applying the best possible combination of contact force $\boldsymbol{\lambda}^*$ yields the minimum motion cost of the object. Despite the simplifications, \eqref{equ:outer_loop} remains a complex problem with nonsmooth constraints $\boldsymbol{p}_{\text{r}} \in \partial \mathcal{O}$ that depend on the object geometry. Meanwhile, \eqref{equ:inner_loop} is a complementary problem with coupled environment forces $\boldsymbol{\lambda}_{\text{env}}$ and robot contact forces $\boldsymbol{\lambda}_{\text{r}}$. The nonconvexity makes online CSO intractable. We next present the solutions to further simplify both \eqref{equ:outer_loop} and \eqref{equ:inner_loop} to improve computational efficiency.

\subsection{Sampling and Exchange Method}\label{sec:sampling&exchange}
The nonsmoothness on the object geometry leads to discontinuities in the gradient of \eqref{equ:outer_loop}. The inherently nonsmooth nature of the problem motivates us to adopt a sampling-based method. To simplify \eqref{equ:outer_loop}, we uniformly sample a set of candidate contact point set $\mathcal{I}_{\text{obj}} = \{\boldsymbol{p}_i\}_{i=1}^{n_s}$ on the object surface $\partial \mathcal{O}$, thus reformulate \eqref{equ:outer_loop} as:
\begin{equation}
    \min_{\boldsymbol{p}_{\text{r}} \in \mathcal{I}_{\text{obj}}}
    \quad V(\boldsymbol{p}_{\text{r}}),\label{equ:simplify_outer_loop}
\end{equation}
When the sampling is sufficiently uniform, \eqref{equ:simplify_outer_loop} provides a good approximation to \eqref{equ:outer_loop} and is not affected by discontinuous gradients caused by the geometry. To obtain a high quality candidate set, $\mathcal{I}_{\text{obj}}$ should satisfy the following conditions:
\begin{itemize}
    \item \textbf{Physical reachability}, which requires that candidate contact locations be reachable by the robot, i.e., within the robot's workspace and not occluded.
    \item \textbf{Contact stability}, which requires avoiding regions with high curvature where stable contact is difficult to apply.
    \item \textbf{Uniform distribution}, which means the sampled points are representative of approximating the original geometry.
\end{itemize}
These conditions require the sampling method to maintain uniform coverage while prioritizing manipulability over local geometric details. To address this, we propose an efficient sampling strategy as shown in Algorithm \ref{algorithm:sampling}. We first perform approximate convex decomposition (ACD) \cite{lien2007approximate} to obtain a convex approximation $\mathcal{F}$ of the original geometry $\mathcal{O}$. ACD decomposes the object into convex polygons, mitigating interference caused by local geometric details and reducing computational cost. Moreover, these polygons are typically larger and flatter, which implies greater contact stability and improved reachability.

Then we perform Farthest-Point Sampling (FPS) on the convex approximation to uniformly sample points on the polygons. Starting from an arbitrary point as a sample, the algorithm iteratively selects the points that maximize the minimal Euclidean distance to all previously chosen samples:
\begin{equation}
p'_{k+1} = \arg\max_{p_i \in \mathcal{F}} \min_{p'_j \in \mathcal{I}_{k}} \|\boldsymbol{p}_i - \boldsymbol{p}'_j\|^2, \label{equ:FPS}
\end{equation}
where $\mathcal{I}_{k}$ is the set of already selected samples. The sampling process is performed offline. Solving \eqref{equ:FPS} yields a uniformly distributed set of points $\mathcal{I}_{\text{obj}}$. For each sampled point $p'_j \in \mathcal{I}_{k}$, we also construct its local coordinate frame to obtain the tangential and normal vectors as local contact information. The local orthonormal frame is constructed as $(\boldsymbol{n}_i, \boldsymbol{t}_{1,i}, \boldsymbol{t}_{2,i})$ in which the normal vector $\boldsymbol{n}_i$ is obtained from the mesh or point-cloud normal estimate, and two orthogonal tangent directions are computed as:
\begin{equation}
\boldsymbol{t}_{1,i} = 
\frac{\boldsymbol{a} - (\boldsymbol{a} \cdot \boldsymbol{n}_i)\boldsymbol{n}_i}
{\|\boldsymbol{a} - (\boldsymbol{a} \cdot \boldsymbol{n}_i)\boldsymbol{n}_i\|}, \ 
\boldsymbol{t}_{2,i} = \boldsymbol{n}_i \times \boldsymbol{t}_{1,i},
\end{equation}
where $\boldsymbol{a}$ is an arbitrary fixed axis. We use this local geometric information to compute the Jacobian required for contact force optimization. Furthermore, a KD-tree $\mathcal{T}=\left \{ (\boldsymbol{p}_i,\boldsymbol{n}_i, \boldsymbol{t}_{1,i}, \boldsymbol{t}_{2,i}) \right \}^{n_s}_{i=1} $ is constructed over the set $\mathcal{I}_{\text{obj}}$ to enable efficient nearest-neighbor queries.

An online Valid Region Selection procedure is further applied to identify feasible and task-relevant surface regions, filtering out unreachable or low-quality points such as occluded regions, regions outside the robot workspace, and semantically unsuitable regions for manipulation. The valid region is maintained by applying a mask to $\mathcal{F}$ to obtain $\mathcal{F}_{\text{v}}$, and then selecting points on $\mathcal{F}_{\text{v}}$ to obtain $\mathcal{I}_{\text{avail}}$. To better illustrate the process, consider the toy example shown in Fig. \ref{fig:FPS}. First, following the procedure illustrated in Fig. \ref{fig:FPS}(a), we iteratively apply \eqref{equ:FPS} to obtain a set of uniformly distributed sampling points over the processed object mesh. We then partition the surface online into a set of valid regions and masked regions, and select the points that lie within the valid regions to form the candidate set $\mathcal{I}_{\text{avail}}$, as shown in Fig. \ref{fig:FPS}(b).

\begin{figure}[!t]
 \centerline{\includegraphics[width=\columnwidth]{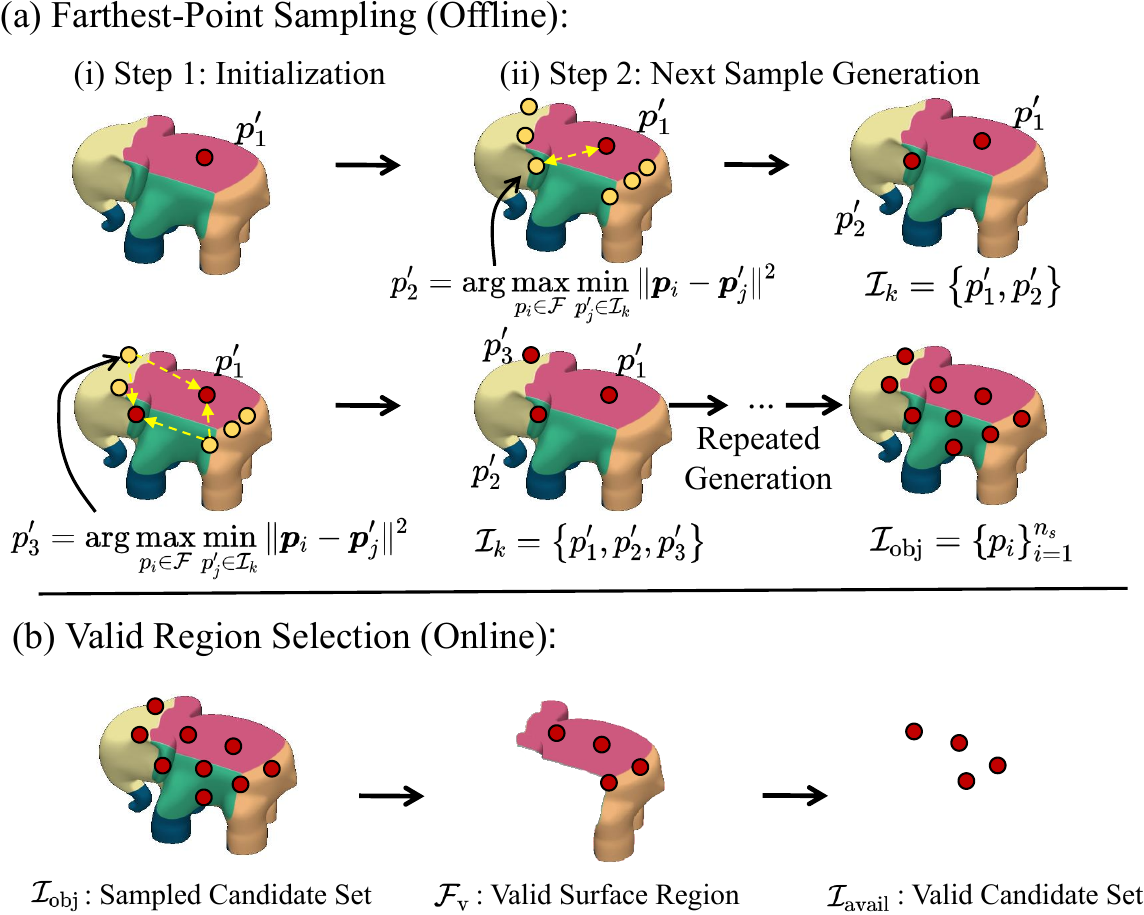}}
\caption{The sampling strategy in CSO. The candidate point set is sampled sequentially through: (a) Farthest-Point Sampling to uniformly sample points on the object, and (b) Valid Region Selection to select valid points for contact.}
\label{fig:FPS}
\end{figure}

\begin{algorithm}[t]
\caption{Sampling and Exchange Method}
\label{algorithm:sampling}
\begin{algorithmic}[1]
\Require Object geometry $\mathcal{O}$, sample count $K$, fixed axis $\boldsymbol{a}$
\Ensure $\mathcal{T}= \left \{ (\boldsymbol{p}_i,\boldsymbol{n}_i, \boldsymbol{t}_{1,i}, \boldsymbol{t}_{2,i}) \right \}^{n_s}_{i=1}$ and $\mathcal{I}_{\text{avail}}$.

\State $\mathcal{F} \gets \text{ACD}(\mathcal{O})$ \Comment{Convex approximation}
\State Initialize $\mathcal{I}_{k} \gets \{p'_1\}$ with an arbitrary seed in $\mathcal{F}$
\For{$k=1$ to $K-1$}
    \State $p'_{k+1} \gets \arg\max_{p'_i\in \mathcal{F}} \min_{p_j\in \mathcal{I}_{k}} \|\boldsymbol{p}'_i-\boldsymbol{p}'_j\|^2$
    \State $\mathcal{I}_{\text{obj}} \gets \mathcal{I}_{k} \cup \{p'_{k+1}\}$ \Comment{Obtain candidate set $\mathcal{I}_{\text{obj}}$}
\EndFor
\State $\mathcal{F}_{\text{v}} \gets \text{getValidRegion}(\mathcal{F})$ \Comment{Task-specific segment}
\ForAll{$p_i \in \mathcal{I}_{\text{obj}}$}
    \If{$p_i \notin \mathcal{F}_{\text{v}}$}
        \State Remove $p_i$ from $\mathcal{I}_{\text{obj}}$ \Comment{Valid region selection}
        \State \textbf{continue}
    \EndIf \Comment{Obtain $\mathcal{I}_{\text{avail}}$}
    \State Estimate normal $\boldsymbol{n}_i$ from $\mathcal{F}$ at $p_i$
    \State $\boldsymbol{t}_{1,i} \gets \dfrac{\boldsymbol{a}-(\boldsymbol{a}\cdot\boldsymbol{n}_i)\boldsymbol{n}_i}{\|\boldsymbol{a}-(\boldsymbol{a}\cdot\boldsymbol{n}_i)\boldsymbol{n}_i\|}$, \ $\boldsymbol{t}_{2,i} \gets \boldsymbol{n}_i \times \boldsymbol{t}_{1,i}$
\EndFor
\State \Return $\mathcal{T}=\left \{ (\boldsymbol{p}_i,\boldsymbol{n}_i, \boldsymbol{t}_{1,i}, \boldsymbol{t}_{2,i}) \right \}^{n_s}_{i=1}$
\end{algorithmic}
\end{algorithm}

\subsection{Surrogate Contact Model}

The sampling strategy transforms the original optimization into a discrete–continuous optimization problem. We further simplify \eqref{equ:inner_loop} to improve efficiency. Most existing contact models are computationally expensive for scenarios that require repeatedly solving \eqref{equ:inner_loop}. Moreover, the complementarity in contact dynamics and the coupling between $\boldsymbol{\lambda}_{\text{env}}$ and $\boldsymbol{\lambda}_{\text{r}}$  make the optimization difficult to converge, leading to physically inconsistent results. To address this issue and enable online contact selection, we present a surrogate contact model (SCM) as a local approximation of the original contact dynamics. The procedure of SCM is illustrated in Fig. \ref{fig:stick1}. The approximation relies on the following assumptions of the object-environment contact:

\begin{assumption}[Fixed Contact Set]
 The  object-environment contact location remains unchanged in each SCM rollout, i.e., no new contacts are created and existing contacts remain.\label{assumption:mode}
\end{assumption}

\begin{assumption}[Multi-Contact Decoupling]
 The contact between the object and the environment can be approximated by independent spring–damper systems, neglecting higher-order dynamic couplings between different contact locations. \label{assumption:decoupling}
\end{assumption}

We reformulate the original contact dynamics in (\ref{equ.quasi_dyn_1}) as:
\begin{equation}
    \epsilon\boldsymbol{M}_o\boldsymbol{v}_o = 
    h(\boldsymbol{\tau}_o+\sum\nolimits_{i=1}^{n_{\text{r}}}{\boldsymbol{J}}_{\text{r},i}^{\top}\boldsymbol{\lambda}_{\text{r},i} + \sum\nolimits_{j=1}^{n_{\text{env}}}{\boldsymbol{J}}_{\text{env},j}^{\top}\boldsymbol{\lambda}_{\text{env},j}),
    \label{eq:simplified}
\end{equation}
and we denote $\boldsymbol{\tilde{b}}_o = h(\boldsymbol{\tau}_o+\sum\nolimits_{i=1}^{n_{\text{r}}}{\boldsymbol{J}}_{\text{r},i}^{\top}\boldsymbol{\lambda}_{\text{r},i})$. Similar to (\ref{equ.dual_prob_relax}), the physically consistent environment contact force can be obtained by solving the following dual problem:
\begin{multline}
\max_{\boldsymbol{\lambda}_{\text{env}}\geq \boldsymbol{0}}\quad  
    -\frac{1}{2} \boldsymbol{\lambda}_{\text{env}}^\top 
    \left(
    \boldsymbol{\tilde{J}}_{\text{env}} \boldsymbol{M}_o^{-1} \boldsymbol{\tilde{J}}_{\text{env}}^\top
    +\boldsymbol{R}
    \right) 
    \boldsymbol{\lambda}_{\text{env}} \\
    -
    (\boldsymbol{\tilde{J}}_{\text{env}} \boldsymbol{M}_o^{-1}\boldsymbol{\tilde{b}}_o+\boldsymbol{\tilde{\phi}})^\top
    \boldsymbol{\lambda}_{\text{env}}
    -\frac{1}{2} 
    \boldsymbol{\tilde{b}}_o^\top
    \boldsymbol{M}_o^{-1}
    \boldsymbol{\tilde{b}}_o,
    \label{eq:dual_problem}
\end{multline}
and the solution satisfies the complementarity condition
\begin{equation}
\boldsymbol{0}\leq \boldsymbol{\lambda}_{\text{env}} \perp
\boldsymbol{W}_{\text{env}}
\boldsymbol{\lambda}_{\text{env}}
+\boldsymbol{g}_{\text{env}}
\geq \boldsymbol{0}.
\label{eq:env_complementarity}
\end{equation}
where
\begin{equation}
    \boldsymbol{W}_{\text{env}}=\left(
\boldsymbol{\tilde{J}}_{\text{env}} \boldsymbol{M}_o^{-1} \boldsymbol{\tilde{J}}_{\text{env}}^\top
+\boldsymbol{R}
\right), \ \boldsymbol{g}_{\text{env}}=\left(
\boldsymbol{\tilde{J}}_{\text{env}} \boldsymbol{M}_o^{-1}\boldsymbol{\tilde{b}}_o
+\boldsymbol{\tilde{\phi}}
\right).
\end{equation}
Here $\boldsymbol{W}_{\text{env}}$ is the Delassus matrix associated with object–environment contact. Since the contact force at each contact point is decomposed into one normal and two tangential components $
\boldsymbol{\lambda}_i = 
[\lambda_i^{n}, \lambda_i^{t_1}, \lambda_i^{t_2}]^\top \in \mathbb{R}^{3}
$, the Delassus matrix $\boldsymbol{W}_{\text{env}}$ admits a block structure
\begin{equation*}
    \boldsymbol{W}_{\text{env}} =
\begin{bmatrix}
\boldsymbol{W}_{11} & \boldsymbol{W}_{12} & \cdots & \boldsymbol{W}_{1n_{\text{env}}} \\
\boldsymbol{W}_{21} & \boldsymbol{W}_{22} & \cdots & \boldsymbol{W}_{2n_{\text{env}}} \\
\vdots & \vdots & \ddots & \vdots \\
\boldsymbol{W}_{n_{\text{env}}1} & \boldsymbol{W}_{n_{\text{env}}2} & \cdots & \boldsymbol{W}_{n_{\text{env}} n_{\text{env}}}
\end{bmatrix},
\end{equation*}
where each block $\boldsymbol{W}_{ij} \in \mathbb{R}^{3\times3}$ is
\begin{equation*}
\boldsymbol{W}_{ij}
=
\begin{bmatrix}
\boldsymbol{J}_i^{n} \\
\boldsymbol{J}_i^{t_1} \\
\boldsymbol{J}_i^{t_2}
\end{bmatrix}
\boldsymbol{M}_o^{-1}
\begin{bmatrix}
(\boldsymbol{J}_j^{n})^\top &
(\boldsymbol{J}_j^{t_1})^\top &
(\boldsymbol{J}_j^{t_2})^\top
\end{bmatrix}
+
\boldsymbol{R}_{ij}.
\end{equation*}
The diagonal blocks $\boldsymbol{W}_{ii}$ describe the local effective inertia at each contact, indicating how forces applied at contact $i$ affect motion along its own constraint directions. The off-diagonal blocks $\boldsymbol{W}_{ij}$ ($i \neq j$) represent dynamic coupling between different contacts. Under Assumption \ref{assumption:decoupling}, these inter-contact couplings are negligible. Therefore, $\boldsymbol{W}_{\text{env}}$ can be approximated by a block-diagonal matrix:
\begin{equation}
\boldsymbol{W}_{\text{env}} = \text{blockdiag}(\boldsymbol{W}_{11}, \boldsymbol{W}_{22}, \ldots, \boldsymbol{W}_{n_{\text{env}} n_{\text{env}}}),
\end{equation}
and \eqref{eq:env_complementarity} is decomposed to subproblems:
\begin{equation}
\boldsymbol{0} \leq \boldsymbol{\lambda}_{\text{env},i} \perp \boldsymbol{W}_{ii} \boldsymbol{\lambda}_{\text{env},i} + (\boldsymbol{\tilde{J}}_{\text{env},i} \boldsymbol{M}_o^{-1}\boldsymbol{\tilde{b}}_o + \boldsymbol{\tilde{\phi}}_i) \geq \boldsymbol{0}.
\label{equ:contact_i_complementarity}
\end{equation}
Furthermore, if $\boldsymbol{W}_{ii}$ admits the low-rank decomposition:
\begin{equation}
\boldsymbol{W}_{ii} = \boldsymbol{D} + \boldsymbol{E}, \label{equ:decomposation}
\end{equation}
where $\boldsymbol{D}$ is a diagonal matrix and
$\boldsymbol{E}$ is relatively small, \eqref{equ:contact_i_complementarity} has the approximated closed-form solution:
\begin{equation}
\boldsymbol{\hat{\lambda}}_{\text{env}, i}
=
\max
\left(
-
\boldsymbol{D}^{-1}
(\boldsymbol{\tilde{J}}_{\text{env},i} \boldsymbol{M}_o^{-1}\boldsymbol{\tilde{b}}_o + \boldsymbol{\tilde{\phi}}_i),
\;\boldsymbol{0}
\right).
\label{eq:quasi-closed_form_solution_nonsmooth}
\end{equation}
With \eqref{eq:quasi-closed_form_solution_nonsmooth}, we can obtain the following lemma:
\begin{lemma}
Under Assumption \ref{assumption:decoupling}, the environment contact force $\boldsymbol{\lambda}_{\text{env}}$ is an explicit function of $\boldsymbol{\lambda}_{\text{r}}$.
\end{lemma}
To obtain a high-quality approximation, we decompose $\boldsymbol{W}_{ii}$ in \eqref{equ:decomposation} as:
\begin{equation}\label{equ:max_decomposition}
\boldsymbol{D} = \operatorname{diag}(\boldsymbol{W}_{ii}) + \epsilon \boldsymbol{I},
\end{equation}
where $\epsilon$ is a small regularization coefficient to ensure numerical stability. \eqref{equ:max_decomposition} gives the closest diagonal approximation to $\boldsymbol{W}_{ii}$ in the Frobenius norm\footnote{Proof see Appendix \ref{proof:closest_approximation}.}. However, it should be noted that \eqref{eq:quasi-closed_form_solution_nonsmooth} is incompatible with smooth NLP solvers. When the contact set changes within the CSO horizon, variations in $\boldsymbol{\tilde{J}}_{\text{env},i}$ and $\boldsymbol{\tilde{\phi}}_i$ introduce additional non-smoothness. To address this, we leverage Assumption \ref{assumption:mode} to further derive:
\begin{lemma}\label{lemma2:piecewise-linear}
The environment contact force $\boldsymbol{\lambda}_{\text{env}}$ is a piecewise linear function of $\boldsymbol{\lambda}_{\text{r}}$:
\begin{equation}
\boldsymbol{\hat{\lambda}}_{\text{env}, i}
=
\max
\left(
-
\boldsymbol{D}^{-1}
\boldsymbol{\tilde{J}}_{\text{env},i} \boldsymbol{M}_o^{-1}(\boldsymbol{\tau}_o+\sum\nolimits_{i=1}^{n_{\text{r}}}{\boldsymbol{J}}_{\text{r},i}^{\top}\boldsymbol{\hat{\lambda}}_{\text{r},i}),
\;\boldsymbol{0}
\right).
\label{eq:quasi-closed_form_solution}
\end{equation}
\end{lemma}
\begin{proof}
    \renewcommand{\qedsymbol}{}
    Based on Assumption 1, we can obtain that $\phi \equiv 0$, leading to \eqref{eq:quasi-closed_form_solution}. Moreover, since the contact geometry set remains invariant over the time interval, $\boldsymbol{\tilde{J}}_{\text{env},i}$ is a constant matrix, thus $\boldsymbol{W}_{ii}$ is constant, and $\boldsymbol{D}$, $\boldsymbol{M}_o^{-1}$ are all constant matrices. Therefore, \eqref{eq:quasi-closed_form_solution} is piecewise linear.
\end{proof}
Substituting \eqref{eq:quasi-closed_form_solution} into \eqref{eq:simplified}, we obtain a lightweight explicit model:
\begin{equation}
    f_{\text{SCM}}:\begin{cases}
    \boldsymbol{\hat{v}}^+ =
\frac{1}{h}\boldsymbol{M}_o^{-1}\boldsymbol{\tilde{b}}_o
+
\frac{1}{h}\boldsymbol{M}_o^{-1}\sum_{i=1}^{n_\text{env}} \boldsymbol{\tilde{J}}_{\text{env},i}^\top
\boldsymbol{\hat{\lambda}}_{\text{env},i},
     \\ 
    \boldsymbol{\hat{x}}_{k+1}=
    \boldsymbol{x}_{k}
    +
    h(\boldsymbol{v}_{k}+\boldsymbol{\hat{v}}^+).
\end{cases} \label{equ:scm}
\end{equation}

This Surrogate Contact Model converts the original contact dynamics into a smooth and computationally efficient formulation, enabling \eqref{equ:inner_loop} to be solved efficiently using standard smooth optimization methods.

\begin{lemma}[SCM error bound]
    Consider the true object–environment contact complementarity system:
    \begin{equation}
        0 \le \boldsymbol{\lambda}_{\text{env}} \perp
\boldsymbol{W}_{\text{env}} \boldsymbol{\lambda}_{\text{env}} + \boldsymbol{g}_{\text{env}} \ge 0,
    \end{equation}
and the velocity change $\boldsymbol{v}^+$:
\begin{equation}
    f_{\text{real}}:\begin{cases}
    \boldsymbol{v}^+ =
\frac{1}{h}\boldsymbol{M}_o^{-1}\boldsymbol{\tilde{b}}_o
+
\frac{1}{h}\boldsymbol{M}_o^{-1}\sum_{i=1}^{n_\text{env}} \boldsymbol{\tilde{J}}_i^\top
\boldsymbol{\lambda}_{\text{env},i},
     \\ 
    \boldsymbol{x}_{k+1}=
    \boldsymbol{x}_{k}
    +
    h(\boldsymbol{v}_{k}+\boldsymbol{v}^+).
\end{cases} \label{equ:scm}
\end{equation}
The SCM  error $\left \| \boldsymbol{x}_{k+1}-\boldsymbol{\hat{x}}_{k+1} \right \|_2 $ is bounded.
\end{lemma}
\begin{proof}
\renewcommand{\qedsymbol}{}
See Appendix \ref{proof:lemma3}.
\end{proof}
\begin{figure}[!t]
 \centerline{\includegraphics[width=\columnwidth]{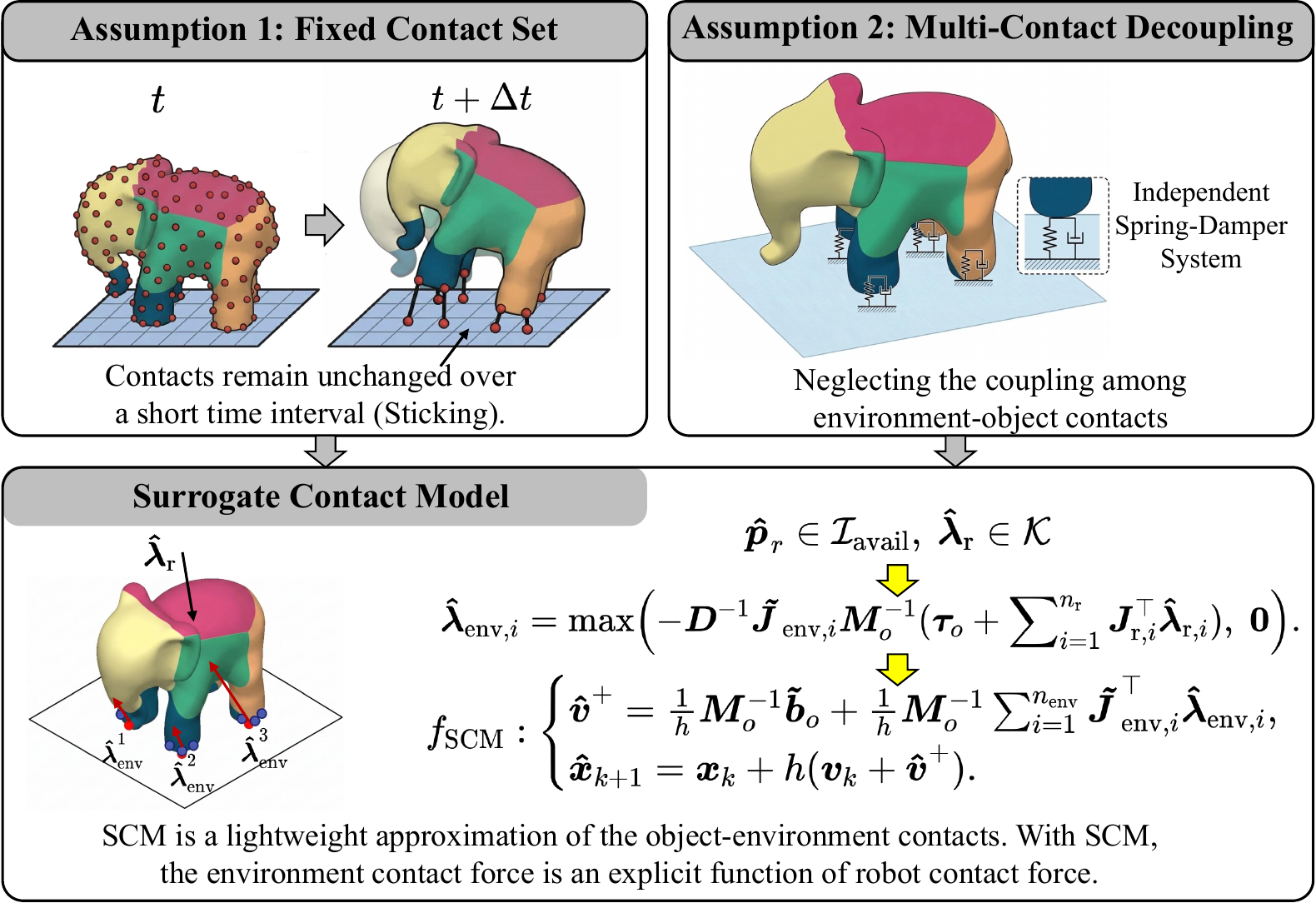}}
\caption{Illustration of SCM. SCM is a lightweight approximation of the object-environment contact, enabling fast contact selection optimization.}
\label{fig:stick1}
\end{figure}
\subsection{Formulation of CSO}\label{sec:cs_opt}

Integrating SCM into \eqref{equ:inner_loop}, we can obtain the formulation of CSO as: 
\begin{equation} 
\begin{aligned} 
\min_{\boldsymbol{\hat{p}}_{\text{r}} \in \mathcal{I}_{\text{avail}}}\min_{\boldsymbol{\hat{\lambda}}_{\text{r}} \in \mathcal{K}} & \quad 
\ell_{\text{cso}}\big(\boldsymbol{x}_{k+1}\big) \\ 
\text{s.t.} \quad
& \boldsymbol{\hat{\lambda}}_{\text{env}, i}
=
\max
\left(
-
\boldsymbol{D}^{-1}
\boldsymbol{\tilde{J}}_{\text{env},i} \boldsymbol{M}_o^{-1}\boldsymbol{\tilde{b}}_o,
\;\boldsymbol{0}
\right), \\
& \boldsymbol{\hat{x}}_{k+1} 
= f_{\text{SCM}}(\boldsymbol{x}_k, \boldsymbol{\hat{p}}_{\text{r}}, \boldsymbol{\hat{\lambda}}_{\text{r}}), 
\quad k = 0, \dots, T-1, \\
& \mu_i \hat{\lambda}^n_{r,i} \ge  \hat{\lambda}^{t1}_{r,i} , \ \mu_i \hat{\lambda}^n_{r,i} \ge  \hat{\lambda}^{t2}_{r,i}, \\
& 0 \le \hat{\lambda}^n_{r,i} \le \hat{\lambda}^n_{\text{max}}, \quad i = 0, \dots, n_{\text{r}}.
\end{aligned}\label{equ:csqp} 
\end{equation} 
By solving this discrete-continuous optimization problem, we can obtain the asymptotically optimal contact location $\boldsymbol{\hat{\lambda}}_r^*$ as guidance for the lower-level contact planning. Next, we discuss the details of CSO.
\subsubsection{Objective}
Despite the objective of CSO being to minimize the object pose cost as denoted in \eqref{equ:obj_pose_cost}, directly using $\ell_{\text{pose}}$ as the CSO objective $\ell_{\text{cso}}$ results in unstable contact strategies by blindly minimizing cost. In addition, the execution of the lower-level contact planning approach is inherently subject to delay, since the robot must physically approach and establish contact rather than instantaneously reaching the selected contact location. Therefore, the design of the CSO should explicitly consider the delay and avoid unstable behaviors that may cause frequent dynamic changes in the object pose. Therefore, we introduce the concept of an equivalence class of stable poses of an object as $\mathcal{X}_{\text{stab}}=\bigcup_{j=1}^{k}\mathcal{X}_{j}$,
where $\mathcal{X}_{j}$ denotes the $j$-th class of stable pose of the object and $\boldsymbol{x}_{\text{o}} \in \mathcal{X}_{j}$ is a statically stable pose at which the object can remain at rest on the table, with its center of mass projected within the support polygon. We define the transfer cost from the current object pose to the goal stable pose as:
\begin{equation}
\ell_{\text{stab}}(\boldsymbol{x}_o)
=
d(\boldsymbol{x}_o,\mathcal{X}_g),
\end{equation}
where $d(\cdot,\cdot)$ is a pose distance metric, $\mathcal{X}_j$ denotes the set of object poses associated with the $j$-th statically stable placement mode, and $\alpha>0$ is a weighting coefficient. We reformulate the CSO objective to prioritize reaching the goal pose set $\mathcal{X}_g$ from the current pose $\boldsymbol{x}_o$, and then minimize $\ell_{\text{pose}}$ in $\mathcal{X}_g$:
\begin{equation}\label{equ:objectice_obj}
\ell_{\text{cso}} =
\begin{cases}
\ell_{\text{stab}}, & \text{if } \boldsymbol{x}_o \notin \mathcal{X}_g, \\
\ell_{\text{pose}}, & \text{otherwise}.
\end{cases}
\end{equation}
The design of \eqref{equ:objectice_obj} effectively guides the gradient descent of CSO to select a more statically stable trajectory. Although this increases the control cost, it effectively prevents CSO from generating numerous unstable contact locations as priors, which would otherwise make the lower-level contact planning difficult to execute in practice.

\subsubsection{Solving CSO as MIQP}
The original contact dynamics in \eqref{equ:inner_loop} are non-differentiable because of the complementarity constraints and the non-smooth temporal evolution of the contact set, rendering the constraints non-differentiable. With the SCM, \eqref{equ:csqp} is formulated as an MIQP. The SCM gradients can be computed via: 
\begin{equation} 
    \begin{aligned} 
    \frac{\partial \boldsymbol{x}_{k+1}}{\partial \boldsymbol{\lambda}_{\text{r},i}} 
&=
\frac{\partial \boldsymbol{x}_{k+1}}{\partial \boldsymbol{v}^{+}} 
\left( 
\frac{\partial \boldsymbol{v}^{+}}{\partial \boldsymbol{b}} 
\frac{\partial \boldsymbol{b}}{\partial \boldsymbol{\lambda}_r} 
+
\frac{\partial \boldsymbol{v}^{+}}{\partial \boldsymbol{\lambda}_{\text{env}}} 
\frac{\partial \boldsymbol{\lambda}_{\text{env}}}{\partial \boldsymbol{\lambda}_r} 
\right) \\ 
&=\boldsymbol{M}_o^{-1}\!\left(h\boldsymbol{J}_{\text{r},i}^\top  
+\frac{1}{h} \sum_{j=0}^{n_\text{env}}\boldsymbol{\tilde{J}}_{\text{env},j}^\top
\frac{\partial \boldsymbol{\lambda}_{\text{env},j}}{\partial \boldsymbol{\lambda}_{\text{r},i}}\right), 
\end{aligned} 
\end{equation} 
where the gradient $\frac{\partial \boldsymbol{\lambda}_{\text{env},j}}{\partial \boldsymbol{\lambda}_{\text{r},i}}$ can be computed via automatic differentiation and satisfies Lipschitz continuity because \eqref{eq:quasi-closed_form_solution} is piecewise linear.  Moreover, the objective function $\ell_{\text{cso}}$ is second-order differentiable. Therefore, the inner loop of \eqref{equ:csqp} constitutes a QP that can be solved in sub-millisecond time, while the overall CSO forms an efficient MIQP compatible with state-of-the-art numerical solvers.

\section{Contact Planning Optimization}\label{sec:cpo}

In this section, we introduce the CPO module for contact planning. CPO accepts the optimal contact location predicted by CSO as reference, and generates approaching and contact trajectories online. We will discuss the problem of designing a prior-guided contact planning. Then we design a ranking strategy as an expressive connection between contact planning and contact prior from contact selection module. Finally, we propose a novel objective formulation for prior-guided contact trajectory generation, and present two solvers for CPO.

\subsection{Problem Statement}

Compared to CIMPC methods, CPO must additionally consider the guidance from the high-level module and generate the corresponding executable trajectory. Based on \eqref{eq:CIMPC_opt}, we define CPO as the single shooting problem:
\begin{equation}
    \begin{aligned}
\min_{\boldsymbol{u}_{0:N-1}} \quad & 
 {\textstyle \sum_{i=0}^{N-1}}  \ell_{\text{cpo}}(\boldsymbol{x}_k, \boldsymbol{u}_k, \boldsymbol{\hat{p}}^*_{\text{r}}) + V_f(\boldsymbol{x}_N) \\
\text{s.t.}\quad & \boldsymbol{x}_{k+1} = f(\boldsymbol{x}_k, \boldsymbol{u}_k), \quad k=0,\dots,N-1, \\
& \boldsymbol{x}_k \in \mathcal{X}, \quad k=0,\dots,N, \\
& \boldsymbol{u}_k \in \mathcal{U}, \quad k=0,\dots,N-1.
\end{aligned}\label{equ:origin_cpo_problem}
\end{equation}
Here $\boldsymbol{\hat{p}}^*_{\text{r}}$ is explicitly included in the objective function. We next investigate two key questions: in what form should contact prior guide contact planning, and how should a feasible approach and contact trajectory be defined mathematically?
\subsubsection{Contact Prior}
In the previous section, we discussed how CIMPC methods suffer from gradient sparsity and proposed CSO, which serves as a pointwise optimization on the object and performs a global search for optimal contact locations. But this raises another open question: in what form should the output of the object-centric contact selection guide be robot-centric contact planning? The most straightforward approach is full coupling, where the contact planning strictly follows and executes the guidance from the CSO. However, this fully coupled scheme reduces robustness in the presence of model inaccuracies: when discrepancies arise between high level planning and low level execution (which are often unavoidable in hierarchical frameworks), the entire framework may fail because of its limited adaptivity. Moreover, although reference contact locations as prior can alleviate the bottleneck of contact planning, optimality under the CSO criterion does not necessarily correspond to the highest probability of successful execution. Contact planning is a hybrid problem admitting multiple solution paths corresponding to different combinations of contact locations, forces, and modes. CSO provides only a single locally cost-minimizing option. To overcome this limitation, we introduce an evaluation module that assesses multiple candidate contact locations and selects the most suitable one as the reference for planning.

\subsubsection{Hybrid Planning of Manipulation Trajectory}\label{sec:feasible_contact_trajectory}

Existing CIMPC methods do not consider the planning of selectively approaching to the reference contact location and switching contact. We define the approaching trajectory in $\mathcal{S}_{\text{free}}$ for CPO to satisfy the following conditions:
\begin{itemize}
\item Except for the start and end points, the mid-points of the trajectory should be collision-free with the object.
\item Contact should occur at the trajectory's end point.

\item The contact force should be directed into the object's interior, i.e., $\lambda_{r,i}^n > 0, \ \forall i = 0, \dots, n_{\text{r}}$.
\end{itemize}
As shown in Fig. \ref{fig:approaching_contact}, the state space of robot-object contact represents a low-dimensional manifold within the overall system state space, while the remaining vast space is gradient-free. The CPO must not only plan trajectories for controlling the object on the contact-rich state manifold $\mathcal{S}_{\text{contact}}$, but also plan trajectories on the bridging manifold $\mathcal{S}_{\text{approach}}$ embedded in $\mathcal{S}_{\text{free}}$. $\mathcal{S}_{\text{approach}}$ must also satisfy the robot’s complex kinematic constraints, which makes it highly non-convex. For redundant robotic configurations, directly solving the nonconvex trajectory optimization online is challenging. We therefore construct an objective formulation for CPO to explicitly represent the processes of approaching, contacting, and switching contacts, enabling unified hybrid planning of the manipulation trajectory.

\subsection{Ranking Strategy of Contact Locations} \label{sec:cpo-ranking_strategy}

In addition to the optimal contact location $\boldsymbol{\hat{p}}^*_{\text{r}}$ from the CSO prediction, we always consider the nearest point $\boldsymbol{p}_{\text{near}}$ to the current robot configuration as a candidate contact location. We reformulate the problem \eqref{equ:origin_cpo_problem} as tracking the reference contact location $\boldsymbol{p}_{\text{ref}}$, where $\boldsymbol{p}_{\text{ref}}$ is generated in real time through a ranking strategy between $\boldsymbol{p}_{\text{near}}$ and $\boldsymbol{\hat{p}}^*_{\text{r}}$, i.e., $\boldsymbol{p}_{\text{ref}} \in \left \{ \boldsymbol{p}_{\text{near}}, \boldsymbol{\hat{p}}^*_{\text{r}} \right \} $. We define the update rule of $\boldsymbol{p}_{\text{ref}}$ to satisfy the following:
\begin{itemize}
    \item When $\boldsymbol{p}_{\text{near}}$ is acceptable under a certain metric, they should be preferentially selected as $\boldsymbol{p}_{\text{ref}}$.
    \item To ensure smooth operation, updates to $\boldsymbol{p}_{\text{ref}}$ should not occur too frequently.
\end{itemize}

The selection of $\boldsymbol{p}_{\text{near}}$ is implemented by searching on the prebuild KD-tree $\mathcal{T}$ with:
\begin{equation}
\boldsymbol{p}_{\text{near}} \;=\; \arg\min_{\boldsymbol{p}_i \in \mathcal{I}_{\text{avail}}}\,\lVert \boldsymbol{p}_{\text{ee}} - \boldsymbol{p}_i \rVert_2, \label{equ:nearest}
\end{equation}
and the local contact frame at the projected point is obtained from $\mathcal{T}$, which is then used to construct constraints during contact trajectory generation.  

The CSO cost of $\boldsymbol{p}_{\text{near}}$, denoted as $\ell_{\text{near}}$, is computed via \eqref{equ:csqp}. Let $\hat{\ell}_{\text{max}}$, $\hat{\ell}_{\text{min}}$ be the maximum and minimum cost of the points in the candidate set $\mathcal{I}_{\text{avail}}$, we define an improvement ratio as: 
\begin{equation}\label{equ:ranking1}
\rho = \frac{\ell_{\text{near}} - \hat{\ell}_{\text{min}}}{\max(\hat{\ell}_{\text{max}}-  \hat{\ell}_{\text{min}}, \varepsilon)},
\end{equation}
where $\varepsilon>0$ is a small constant for numerical stability and $\ell_{\text{near}}$ is the contact cost of $\boldsymbol{p}_{\text{near}}$. The improvement ratio measures $\ell_{\text{near}}$ on a normalized scale, which provides a balanced criterion of the quality of $\boldsymbol{p}_{\text{near}}$. Next, we define a trigger variable $\gamma \in\left \{ 0,1 \right \}$ to determine whether to switch the contact location as 
\begin{equation}\label{equ:kappa}
    \kappa_t = \mathbf{1}\!\left(\rho_t > \bar{\rho}\right),
\end{equation}
where $\bar{\rho}$ is a threshold and $\mathbf{1}$ is a boolean trigger function. When $\kappa_t=1$, it indicates that $\boldsymbol{p}_{\text{near}}$ is a sufficiently good suboptimal reference contact location where manipulation can be executed. If $\kappa_t = 0$, the reference contact location $\boldsymbol{p}_{\text{ref}}$ should continue to be set as $\boldsymbol{\hat{p}}^*_{\text{r}}$ until a sufficiently good $\boldsymbol{p}_{\text{near}}'$ is identified along the approaching path to $\boldsymbol{\hat{p}}^*_{\text{r}}$. The trigger variable $\kappa$ is defined to evaluate the feasibility of $\boldsymbol{p}_{\text{near}}$ and to control whether to establis or break of contact with the object.
To avoid frequent contact location switching, we do not switch it directly based on $\kappa$ but instead apply the following sliding window strategy:
\begin{equation}\label{equ:ranking}
\gamma_{t+1} =
\begin{cases}
1, & \kappa_t = 0 \ \wedge\ t_{(\kappa=1)} \ge T_1,\\
0, & \kappa_t = 1 \ \wedge\ t_{(\kappa=1)} \ge T_2 \ \wedge\ t_{\text{contact}} \ge T_3,\\
\gamma_t, & \text{otherwise}.
\end{cases}
\end{equation}
in which $t_{(\kappa=1)}$ is the total timestep that satisfies $\kappa = 1$. $\gamma$ subsequently controls the objective formulation in CPO.

The overall process is therefore: the robot repeatedly checks whether nearby contact is feasible at the current configuration. If true, $\boldsymbol{p}_{\text{near}}$ will be set as the new reference contact point $\boldsymbol{p}_{\text{ref}}$. Otherwise, the robot continues moving toward the optimal contact location $\boldsymbol{\hat{p}}^*_{\text{r}}$ predicted by CSO. The process is shown in Algorithm \ref{algo:ranking}. The ranking strategy provides a comprehensive evaluation of contact priors to determine their suitability for execution. It allows the generation of sub-optimal manipulation trajectories under feasible conditions, thereby further enhancing the diversity of manipulation strategies. This approach effectively improves robustness in complex contact-rich manipulation tasks.

\begin{algorithm}[t]
\caption{Ranking Strategy of CPO}
\label{algo:ranking}
\begin{algorithmic}[1]
\Require $\boldsymbol{\hat{p}}^*_r$ predicted by CSO, current robot configuration $\boldsymbol{q}_k$, threshold $\delta$
\Ensure Selected contact location $\boldsymbol{p}_{\text{ref}}$ and $\ell_{\text{lp}}$
\State Compute nearest point $\boldsymbol{p}_{\text{near}}$ to end-effector via \eqref{equ:nearest}
\State Evaluate contact cost $\ell_{\text{near}}$ via \eqref{equ:csqp}
\State Compute $\gamma$ via \eqref{equ:ranking1},  \eqref{equ:kappa}, \eqref{equ:ranking} \Comment{Ranking strategy}
\If{$\gamma =0$}
    \State $\boldsymbol{p}_{\text{ref}} = \boldsymbol{\hat{p}}^*_{\text{r}}$, $\ell_{\text{lp}}=\ell_{\text{lift}}$ \Comment{Lifting phase}
\Else
    \State $\boldsymbol{p}_{\text{ref}} \gets \boldsymbol{p}_{\text{near}}$, $\ell_{\text{lp}}=\ell_{\text{place}}$ \Comment{Placing phase}
\EndIf 
\end{algorithmic}
\end{algorithm}

\subsection{Formulation of CPO} 

Prior-guided contact planning should be capable of hybrid planning for the manipulation trajectory, including approaching, contacting, and switching between different contact locations. This process is incompatible with CIMPC formulation because it requires the robot to actively leave the $\mathcal{S}_{\text{contact}}$ and enter a gradient-free state space $\mathcal{S}_{\text{free}}$. To address this issue, we propose a simple but efficient design of the objective formulation that enables hybrid planning in a unified CIMPC framework. The proposed formulation is applicable to a wide range of task settings and compatible with different solvers.

As illustrated in Fig. \ref{fig:task_requirements}, we decompose the trajectory optimization into three stages. In the first stage, the end-effector is guided away from the object and moves around the object toward an intermediate waypoint $\boldsymbol{p}_{\text{lift}}$ or configuration $\boldsymbol{q}_{\text{lift}}$. This intermediate waypoint or configuration serves to simplify the computation of collision-free approaching trajectories for the robot.  In the second stage, the end-effector is driven to $\boldsymbol{p}_{\text{ref}}$ to establish contact. This design is inspired by footstep planning in legged robots \cite{acosta2025perceptive}. We next provide the detailed design of the relevant costs.

\subsubsection{Lifting Phase} In the lifting phase, the currently nearest contact point $\boldsymbol{p}_{\text{near}}$ is evaluated to be infeasible for contact. A hybrid potential is formed to guide the robot to break contact and move toward the region above the desired contact location. The potential consists of a repulsive potential from the object:
\begin{equation}
    \ell_{\text{obs}} = 
\begin{cases}
\log\left ( {\|\boldsymbol{p}_{\text{ee}} - \boldsymbol{p}_{\text{o}}\|^2 + \varepsilon} \right ) , & 
\text{if } \|\boldsymbol{p}_{\text{ee}} - \boldsymbol{p}_{\text{o}}\|^2 < \sigma, \\[8pt]
0, & \text{otherwise.}
\end{cases}
\end{equation}
and an attractive potential from lifting point $\boldsymbol{p}_{\text{lift}}$ or lifting pose $\boldsymbol{q}_{\text{lift}}$. The lifting point/pose satisfies:
$\left \| \boldsymbol{p}_{\text{lift}}-\hat{\boldsymbol{p}}_r^* \right \|_2  = d $ or $ \left \|\text{FK}({\boldsymbol{q}_{\text{lift}})}-\hat{\boldsymbol{p}}_r^*\right \|_2  = d$,
 which are responsible for attracting the robot to the location above $\hat{\boldsymbol{p}}_r^*$. The attractive potential is defined as:
\begin{equation}
    \ell_{\text{att}} = \|\boldsymbol{p}_{\text{ee}} - \boldsymbol{p}_{\text{lift}}\|^2 \quad \text{or} \quad \ell_{\text{att}} = \|\boldsymbol{q}_{\text{ee}} - \boldsymbol{q}_{\text{lift}}\|^2.
\end{equation}
The lift phase potential field cost is written as:
\begin{equation}
    \ell_{\text{lift}} = w_{\text{att}} \, \ell_{\text{att}} + w_{\text{obs}} \, \ell_{\text{obs}}.
\end{equation}
The potential field cost explicitly encodes the ``lifting" behavior, avoiding the impact of gradient-free regions on trajectory optimization during contact breaking.

\subsubsection{Placing Phase}
In the placing phase, when the robot’s end-effector reaches a position above $\hat{\boldsymbol{p}}_r^*$, or when $\boldsymbol{p}_{\text{near}}$ is evaluated as an acceptable contact location, the robot should make contact with the object. The placing phase is also achieved by a potential field cost:
\begin{equation}
    \ell_{\text{place}} = w_{\text{con}}  \ell_{\text{con}} + w_{\text{o}} \ell_{\text{o}},
\end{equation}
in which $\ell_{\text{con}}$ is the attracting potential from the reference contact point $\boldsymbol{p}_{\text{ref}}$, and $\ell_{\text{o}}$ is the attracting potential from the object center:
\begin{equation}
    \ell_{\text{con}} = \log \ \left(\|\boldsymbol{p}_{\text{ee}} - \boldsymbol{p}_{\text{o}}\|^2 + \varepsilon\right), \ \ell_{\text{o}} = \left \| \boldsymbol{p}_{\text{o}} - \boldsymbol{p}_{\text{ee}} \right \|^2.
\end{equation}
Relying solely on $\ell_{\text{con}}$ for the placing phase may fail to ensure actual contact between the robot and the object when the object mesh is inaccurate, as perception or projection errors can cause $\hat{\boldsymbol{p}}_r^*$ to deviate from the object surface. The introduction of $\ell_{\text{o}}$ mitigates this issue and makes the placement process more robust.

\subsubsection{Switching Parameter} The switching between lifting and placing behaviors is controlled by the trigger $\gamma$ in Section \ref{sec:cpo-ranking_strategy}. The overall lift-and-place cost is defined as:
\begin{equation}
    \ell_{\text{lp}} = (1-\gamma)  \ell_{\text{lift}} + \gamma \ell_{\text{place}},
\end{equation}
$\ell_{\text{lp}}$ serves as a part of the stage cost $\ell_{\text{cpo}}$ in \eqref{equ:origin_cpo_problem}. By defining the terminal cost $V_f$ the same as in \eqref{equ:obj_pose_cost}, we can obtain the formulation of CPO as:
\begin{equation}\label{equ:cpo_formulation}
    \begin{aligned}
\min_{\boldsymbol{u}_k} \quad & 
 {\textstyle \sum_{i=0}^{N-1}}  \left (w_{\text{lp}} \ell_{\text{lp}}(\boldsymbol{x}_k, \boldsymbol{q}_k, \boldsymbol{u}_k, \boldsymbol{p}_{\text{ref}}) +w_{u}\left \| \boldsymbol{u}_t \right \|^2 \right ) +w_{\text{cso}}\ell_{\text{cso}} \\
\text{s.t.}\quad & \boldsymbol{x}_{t+1} = f_{\text{cf}}(\boldsymbol{x}_t, \boldsymbol{u}_t), \quad k=0,\dots,N-1, \\
& \boldsymbol{x}_k \in \mathcal{X}, \quad k=0,\dots,N, \\
& \boldsymbol{u}_k \in \mathcal{U}, \quad k=0,\dots,N-1.
\end{aligned}
\end{equation}

\begin{figure}[!t]
 \centerline{\includegraphics[width=\columnwidth]{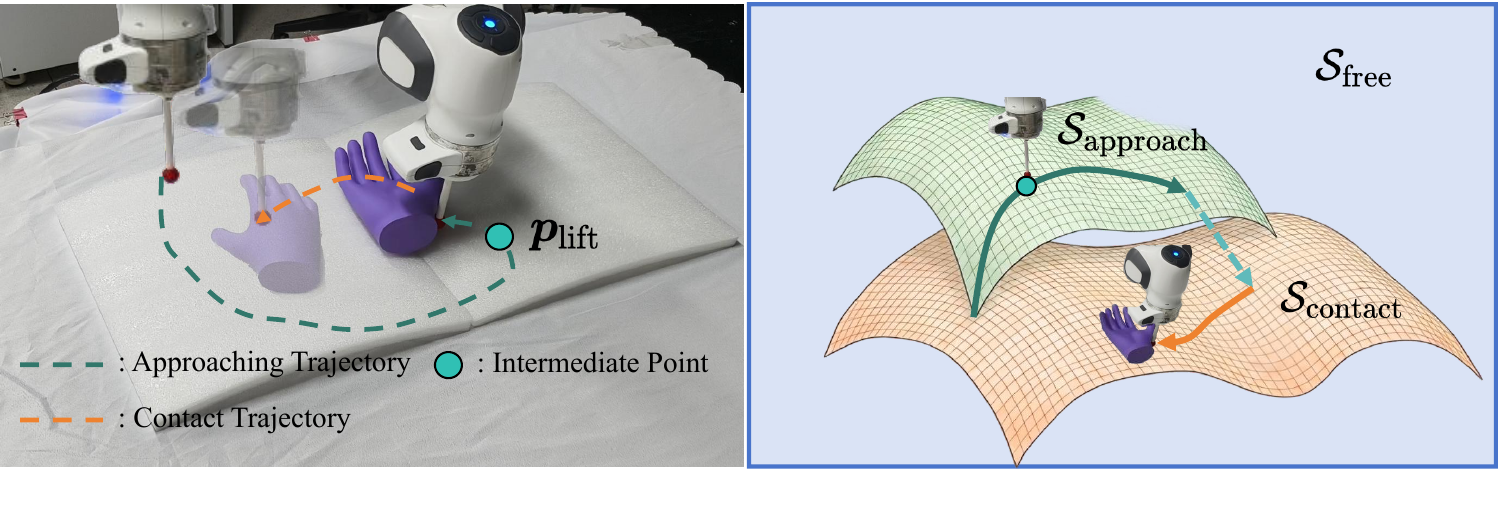}}
\caption{(a) Illustrative example of a feasible approach and contact trajectory. (b) Explanation of CPO in a geometric perspective. CPO can be characterized as operating over two manifolds: contact states in $\mathcal{S}_{\text{contact}}$ and  approaching trajectories on the manifold $\mathcal{S}_{\text{approach}}$ embedded in the contact-free state space $\mathcal{S}_{\text{free}}$. CPO must not only generate contact trajectories on $\mathcal{S}_{\text{contact}}$ but also address the gradient-sparse issue and generate trajectories on $\mathcal{S}_{\text{approach}}$.}
\label{fig:approaching_contact}
\end{figure}

\begin{figure}[!t]
 \centerline{\includegraphics[width=\columnwidth]{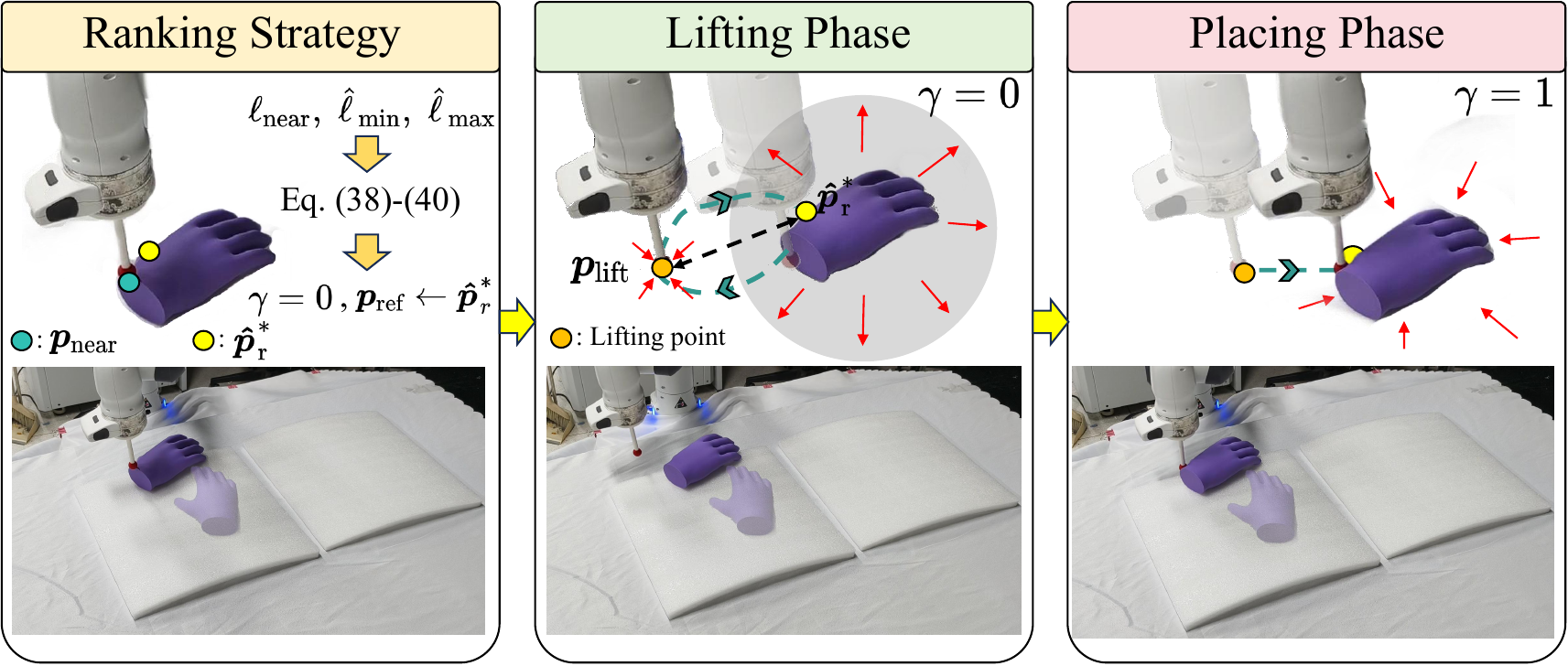}}
\caption{Illustration of the CPO-based contact planning module. (a) Ranking strategy. When $\boldsymbol{p}_{\text{near}}$ is evaluated infeasible, $\gamma$ turns to zero and $ \boldsymbol{p}_{\text{ref}} \gets \boldsymbol{\hat{p}}_r^*$. The evaluation runs continuously in real-time. (b) When $\gamma=0$,  the robot is guided by the potential field to $\boldsymbol{p}_{\text{lift}}$. (c) Once above $\hat{\boldsymbol{p}}_r^*$ or the nearest point $\boldsymbol{p}_{\text{ref}}$ to the robot is detected feasible via \eqref{equ:csqp}, $\gamma$ becomes 1 and an attractive potential is applied, guiding the robot to contact with the object.}
\label{fig:task_requirements}
\end{figure}
The objective formulation in \eqref{equ:cpo_formulation} preserves the favorable properties of the original CIMPC objective. Therefore, CPO can be interpreted as a prior-guided contact planning method. We state the following lemma.
\begin{lemma}[Compatibility of $\ell_{\text{lp}}$ and $\ell_{\text{cso}}$]
$\ell_{\text{lp}}$ and $\ell_{\text{cso}}$ are non-conflicting objectives in CPO: minimizing the lift-and-place objective does not obstruct first-order improvement of the object manipulation objective.
\end{lemma}
\begin{proof}
    \renewcommand{\qedsymbol}{}
    See Appendix \ref{proof:lemma4}.
\end{proof}
\subsection{Solvers of CPO} 

Since CPO specifies only the problem structure without prescribing a particular solver, any reactive motion planner can be employed to generate a contact trajectory. Below, we present two kinds of solvers used in our experiments.

\subsubsection{Operational-Space Control}
\label{sec:osc_control}
We simplify the full-order contact planning problem to an end-effector point contact planning problem, avoiding high-DoF robot trajectory optimization. The resulting joint torques are generated by operational space control (OSC) \cite{nakanishi2008operational} to track the CPO output. Specifically, the state in \eqref{equ:cpo_formulation} is reduced to $\boldsymbol{s} \in \mathbb{R}^{10}$, consisting of the object position, object orientation represented by a quaternion, and the end-effector position $\boldsymbol{p}_{\text{ee}} \in \mathbb{R}^3$. Solving \eqref{equ:cpo_formulation} yields the control input $\boldsymbol{u} = \Delta \boldsymbol{p}_{\text{ee}}$.

Given the desired pose $(\boldsymbol{p}_d, \boldsymbol{R}_d)$ with $\boldsymbol{p}_d=\boldsymbol{p}_{\text{ee}}+\boldsymbol{u}$ and a preset $\boldsymbol{R}_d$, the position and orientation errors are computed as $\boldsymbol{e}_{\text{pos}} = \boldsymbol{u},
$ and $\boldsymbol{e}_{\text{R}} = \boldsymbol{R}_{\text{ee}}^\top \boldsymbol{R}_d$.
Converting $\boldsymbol{e}_{\text{R}}$ to quaternion form gives $\boldsymbol{e}_{\text{quat}} =
\begin{bmatrix}
\boldsymbol{\epsilon}_e, 
\eta_e
\end{bmatrix}^\top
\in \mathbb{R}^4$,
and after sign correction,
\begin{equation}
\boldsymbol{e}_{\text{quat}} \leftarrow
\begin{cases}
\boldsymbol{e}_{\text{quat}}, & \eta_e \ge 0,\\
-\boldsymbol{e}_{\text{quat}}, & \eta_e < 0,
\end{cases}
\qquad
\boldsymbol{e}_{\text{rot}} = -\boldsymbol{R}\boldsymbol{\epsilon}_e.
\end{equation}
By defining
$\boldsymbol{e}_{\text{pose}} =
\begin{bmatrix}
\boldsymbol{e}_{\text{pos}},
\boldsymbol{e}_{\text{rot}}
\end{bmatrix}^\top$,
the task-space and null-space torques are given by:
\begin{align}
\boldsymbol{\tau}_{\text{task}} &=
\boldsymbol{J}^{\top}(\boldsymbol{q})
\left(
-\boldsymbol{K}_x \boldsymbol{e}_{\text{pose}}
-\boldsymbol{D}_x \dot{\boldsymbol{x}}
\right), \\
\boldsymbol{\tau}_{\text{null}} &=
\left(
\boldsymbol{I} - \boldsymbol{J}^{\top}(\boldsymbol{J}^{\top})^{\dagger}
\right)
\left(
\boldsymbol{K}_{q}(\boldsymbol{q}_{\text{init}} - \boldsymbol{q})
- 2\sqrt{\boldsymbol{K}_{q}}\,\dot{\boldsymbol{q}}
\right),
\end{align}
and the final control torque is
$\boldsymbol{\tau} =
\boldsymbol{\tau}_{\text{task}}
+ \boldsymbol{\tau}_{\text{null}}
+ g(\boldsymbol{q})$. This controller is used in all simulations and real-world experiments in Section \ref{sec:experiment}, and is particularly suitable for tasks such as tabletop manipulation, where high orientation diversity is not critical.

\subsubsection{Sampling-based Motion Planner}

Another formulation is to directly solve the full-order planning problem of the robot. In this case, the state in \eqref{equ:cpo_formulation} is given by $\boldsymbol{x} \in \mathbb{R}^{7 + n_{\text{DoF}}}$, where $n_{\text{DoF}}$ denotes the number of robot DoFs. For this class of problems, we employ sampling-based MPC methods as the solver, such as MPPI \cite{williams2015model}. Recent studies  have demonstrated the impressive performance of MPPI in solving high-dimensional optimization problems for redundant robots, enabling real-time computation while avoiding issues related to gradient-based methods. The cost function is defined in the same manner as in the previous formulation, except that $\boldsymbol{p}_{\text{ee}}$ is computed via forward kinematics for the full-order problem.

%% file: experiment.tex
\section{Experiment: Nonprehensile Manipulation}
\label{sec:experiment}

In this section, we demonstrate the effectiveness of the SCSP framework through extensive experimental results. We first conduct nonprehensile manipulation experiments, which represent a complex contact-rich manipulation task with high demands on operational diversity and control robustness. The task is to adjust the pose of an object on the ground through nonprehensile manipulation, which consists of:
\begin{itemize}
    \item \textbf{On-ground rotation:} The robot needs to move the object to the target pose on the ground through pushing and rotating, which constitutes a 2D planar pushing task.
    \item \textbf{On-ground flipping:} The robot needs to move the object to the target pose on the ground through pushing and flipping, which constitutes a 6D manipulation task.
\end{itemize}
We select a set of complex-shaped objects from the ContactDB dataset \cite{brahmbhatt2019contactdb} for validation. The objects are initialized at random poses with 
$x_{\text{init}}, y_{\text{init}} \in [-0.3, 0.3]$ m, 
and Euler angles $(\phi_{\text{init}}, \theta_{\text{init}} , \psi_{\text{init}})$ 
uniformly sampled from $[-\pi, \pi]$ rad. We define the conditions of task success as: 
\begin{equation}
    \left\| \boldsymbol{p}_o - \boldsymbol{p}_{\text{goal}} \right\|
\le 0.02 \,\text{[m]}, \ 
1 - (\boldsymbol{\psi}_{\text{goal}}^\top \boldsymbol{\psi}_o)^2
\le 0.05, \label{eq:pose_success}
\end{equation}
A trial is considered successful if the conditions are satisfied within 2500 rollout steps. We evaluate the algorithm on different robot platforms, namely a fingertip\footnote{The ``fingertip'' is a floating point with a collision volume and is commonly used as a simplified setup for validating contact planning methods.} and the Franka Emika Panda, and compare different methods in terms of success rate, efficiency, and convergence error. This experiment demonstrates that SCSP can autonomously decide contact locations correctly and stably complete the task in complex multi-contact scenarios. 

Simulations are conducted in \texttt{MuJoCo} \cite{todorov2012mujoco} and  \texttt{IsaacGym} \cite{makoviychuk2021isaac} with a fixed integration time step $h=0.02 \text{s}$. For the sampling method in our CSO module, we use \texttt{CoACD} \cite{wei2022approximate} for convex decomposition of the objects and \texttt{trimesh} for mesh processing and local frame construction. The inner loop of CSO is implemented by solving the nonlinear SQP optimization using \texttt{CasADi} \cite{Andersson2019} with \texttt{SNOPT} \cite{gill2005snopt} solver, and the solver of CPO is \texttt{IPOPT} \cite{wachter2006implementation}. All the experiments are done on a desktop computer with the Intel i7-13620H processor and 32 GB RAM.

We applied certain randomization to the environment and object parameters. The friction coefficients of the object $\mu_{\text{o}}$, fingertip $\mu_{\text{r}}$, and ground $\mu_{\text{env}}$ are all set to 0.5 with random perturbations of $\pm 0.2$. The object mass is set to 0.1 kg and randomly scaled within a factor of 10. We set $\boldsymbol{M}_o=\text{diag}(50,50,50,0.05,0.05,0.05)$ and $\boldsymbol{K}_r=300 \boldsymbol{I}_7$, and the dynamic parameters used in SCM are not measured precisely, but are instead roughly set to their mean values. For CSO, we set $n_s=70$ and $\hat{\lambda}^n_{\text{max}}=0.2$. For CPO, we set $\bar{\rho}=0.75$, $T_1=5$, $T_2=10$, $T_3=15$, $\varepsilon=0.01$, $w_{\text{lp}}=1$, $w_{\text{pos}}=500$, $w_{\text{quat}}=5$, $w_\text{u}=50$, $w_{\text{obj}}=10$. We conducted parameter tuning via grid search over 100 randomly sampled trials, and the resulting parameters were fixed for all subsequent experiments. In addition, based on the definition of the object’s local coordinate frame and the planar environment in the experiments, the stable pose set $\mathcal{X}_g$ can be written as $\mathcal{X}_g=\left\{\boldsymbol{x}_o\in\mathcal{X}\;\middle|\;\boldsymbol{R}\left(\boldsymbol{x}_o\right)e_z=e_z\right\}
$ with $e_z=[0, 0, 1]^\top$. Therefore, $\ell_{\text{stab}}$ is defined as:
\begin{equation}
    \ell_{\text{stab}}=1-\bigl(\boldsymbol{R}(\boldsymbol{x}_o)e_z\bigr)^\top\bigl(\boldsymbol{R}(\boldsymbol{x}_{\text{goal}})e_z\bigr).
\end{equation}

From Table \ref{tab:comparison}, we select the methods that are applicable to 6D manipulation, operate in real time, and can handle arbitrary manipulable object geometries as baselines, including:
\begin{itemize}
    \item \textbf{DyWA} \cite{lyu2025dywa}:  An end-to-end 6D nonprehensile manipulation with generalization across diverse unseen objects. We follow the settings in their released source code by adding the test objects to the \texttt{IsaacGym}, and evaluate the performance of the single-view student policy. 
    \item \textbf{Sampling} \cite{venkatesh2025approximating}: A class of hierarchical frameworks  in which the upper layer selects a set of candidate contact points by sampling on the object envelope. The upper layer performs contact selection by evaluating the CIMPC cost and control cost in parallel, while the lower layer uses OSC to achieve approach and contact. To eliminate the effect of differences in contact modeling, we replace the original C3 model \cite{aydinoglu2024consensus} in \cite{venkatesh2025approximating} with $f_{\text{cf}}$.
    \item \textbf{I-MPC:} A CIMPC method based on a QP-based contact model. This method performs contact simulation by solving the KKT conditions of the contact model and implicitly plans the contact mode, force, and location by optimizing robot-centric control inputs. The stage cost is defined as the sum of a contact cost and a control cost $\ell = \left \| \boldsymbol{p}_{\text{obj}} - \boldsymbol{p}_{\text{ee}} \right \|^2 +\left \| \boldsymbol{u} \right \|^2$
and the terminal cost is defined the same as \eqref{equ:obj_pose_cost}. Although the contact planning formulation is unified and does not use hierarchical solving for different contact modes, the planning horizon is short to enable online computation. In addition, the method relies on explicitly defined contact costs to encourage robot–object interaction. As a result, it can only adjust the contact location to a  limited extent.
    \item \textbf{CF-MPC} \cite{jin2024complementarity}: A complementarity-free MPC method based on closed-form contact model \eqref{equ:cf_model}. The cost design is the same as in I-MPC, but it uses an improved contact model, achieving higher manipulation success rates and better computational efficiency. However, its underlying formulation remains unchanged, so its contact selection capability is still limited.
    \item \textbf{A-MPC:} We augment CF-MPC by adding an additional alignment cost into the stage cost:
\begin{equation}
\ell_{\text{align}}
=
-\frac{(  \vec{\mathbf{v}} _{\text{obj}}^\text{goal} )^\top
\vec{\mathbf{v}} _{\text{ee}}^\text{obj} + 1}{2},
\end{equation}
in which:
$
\vec{\mathbf{v}} _{\text{obj}}^\text{goal}=\frac{\boldsymbol{p}_{\text{goal}} - \boldsymbol{p}_{\text{obj}}}{\left \| \boldsymbol{p}_{\text{goal}} - \boldsymbol{p}_{\text{obj}} \right \| }, \ 
\vec{\mathbf{v}} _{\text{ee}}^\text{obj}=\frac{\boldsymbol{p}_{\text{obj}} - \boldsymbol{p}_{\text{ee}}}{\left \| \boldsymbol{p}_{\text{obj}} - \boldsymbol{p}_{\text{ee}} \right \| }
$.
 This cost is motivated by a simple intuition: when pushing an object toward a given direction, the optimal contact location often lies on the side opposite to the pushing direction, since the surface is more likely to contain the desired force within the friction cones. Although this heuristic may not hold for objects with complex geometries, it serves as a simple prior that can guide contact planning methods to escape local minima.
    \item \textbf{w/o RS:} A variant of SCSP without the Ranking Strategy in CPO. In this method, CSO and CPO are fully connected, which means that CPO completely follows and executes the contact locations predicted by CSO. This ablation demonstrates that the ranking strategy is a crucial factor for the robustness of SCSP. 
\end{itemize}

\subsection{Fingertip Manipulation in Simulation}\label{sec:fingertip_exps}

We first conduct fingertip manipulation to validate the critical role of contact selection in contact-rich manipulation. Since DyWA is an offline-trained policy, it is not applicable to the fingertip experiments and is therefore omitted. By comparing with other model-based methods, we demonstrate that explicitly introducing a contact selection module is effective for handling complex manipulation tasks.

\begin{figure}[!t]
 \centerline{\includegraphics[width=\columnwidth]{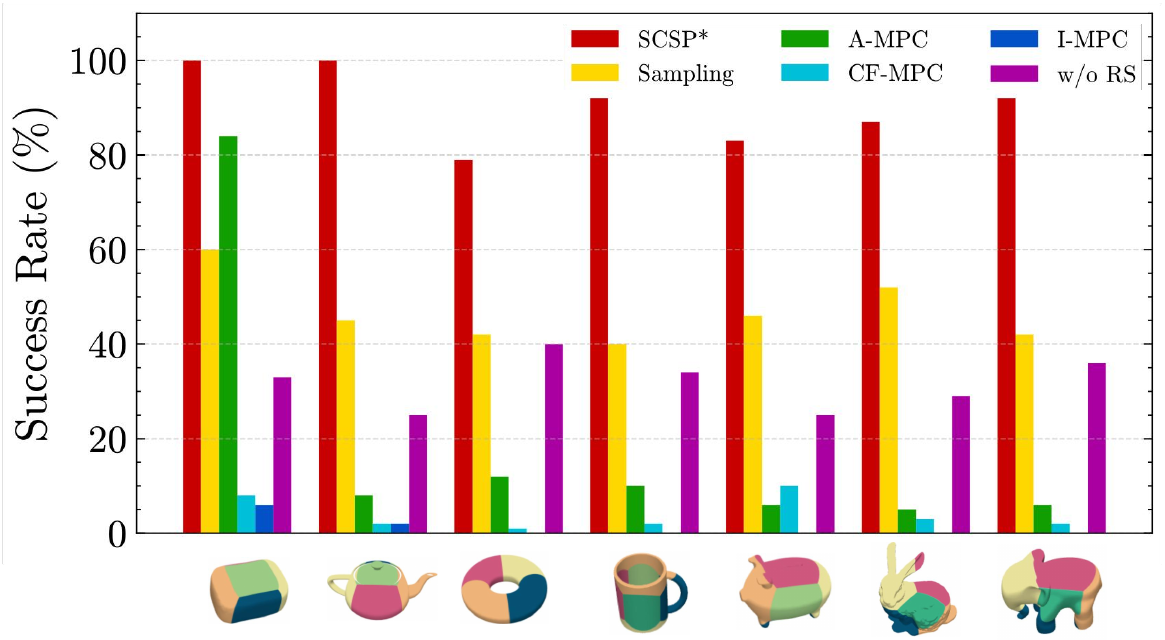}}
\caption{Experimental results of fingertip manipulation. SCSP achieves a substantially higher success rate than the other methods.}
\label{fig:results_fingertip}
\end{figure}

The experimental results are shown in Fig. \ref{fig:results_fingertip}, \ref{fig:simulation_data} and Table \ref{tab:all_data}. Compared with the baseline methods, SCSP achieves a significantly higher success rate, shorter execution time, and lower final position and quaternion errors. Our observations indicate that most of the successful cases of I-MPC and CF-MPC correspond to a good robot initial state, where the methods can succeed by applying contact forces at a fixed contact location. When the fingertip initial position is poor, they exhibit almost no adjustment capability. A-MPC achieves a significant improvement in success rate, demonstrating that even a coarse contact selection prior can substantially mitigate the local optimum problem in contact planning. Sampling achieves further improvement but at the cost of lower computational efficiency. These experimental results highlight the importance of active contact selection in complex tasks.

\subsection{Robot Manipulation in Simulation}\label{sec:robot_manipulation}

We further evaluate SCSP in redundant robot manipulation tasks. Compared with floating-based fingertip, contact problems for redundant robots introduce complex kinematic constraints. In the experiments, we demonstrate the superiority of SCSP and validate the soundness of our framework.

\begin{figure}[!t]
 \centerline{\includegraphics[width=\columnwidth]{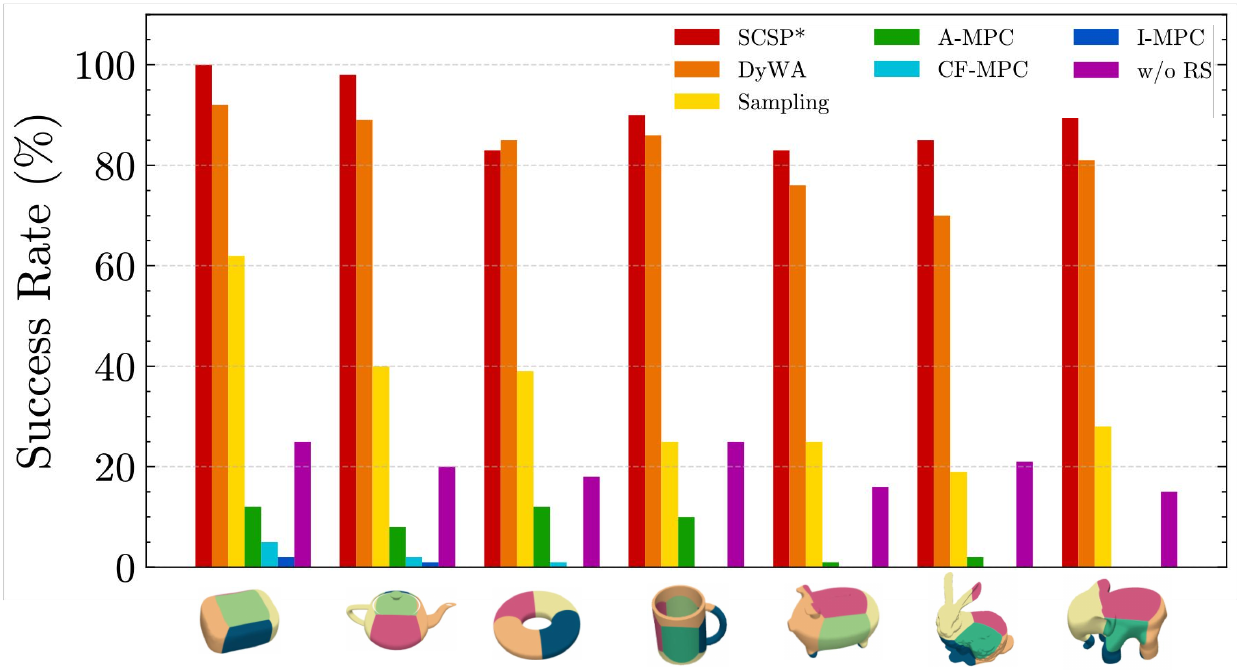}}
\caption{Experimental results of robot manipulation in simulation. SCSP significantly outperforms other model-based methods and exceeds the performance of well-trained learning-based methods.}
\label{fig:results_panda}
\end{figure}

Due to the strong dependencies of DyWA’s source code, which are difficult to transfer, we follow its setup and conduct all experiments in \texttt{IsaacGym} for consistency, with other settings similar to Section \ref{sec:fingertip_exps}. All methods use our link-mounted sphere as the end-effector to replicate point contacts. Since DyWA uses the Franka gripper, we adjust the link length so that the sphere’s center coincides with the gripper’s center. The cartesian stiffness $\boldsymbol{K}_x=\text{diag}(1000,1000,1000,50,50,50)$ and $\boldsymbol{D}_x=\sqrt{2\boldsymbol{K}_x}$. All methods use the same settings as in the fingertip experiments, except that $\boldsymbol{u} \in \mathbb{R}^7$ in I-MPC, CF-MPC and A-MPC. The $\boldsymbol{p}_{\text{ee}}$ is obtained via differentiable forward kinematics.

\begin{table*}[t]
  \begin{center}
  \caption{
  {Comparison of nonprehensile manipulation tasks of the fingertip experiment and the robot experiment (200 trials).}
}
 \label{tab:all_data}
  
  \resizebox{\linewidth}{!}{
  \begin{tabular}{c|c|cc|cc|cc|cc|cc}
    \toprule[1pt]
    \multirow{2}{*}{\textbf{Platform}}
    & \multirow{2}{*}{\textbf{Method}}
    & \multicolumn{2}{c|}{\textbf{Success Rate}}
    & \multicolumn{2}{c|}{\textbf{Execution Time} (s)}
    & \multicolumn{2}{c|}{\textbf{Planning Time} (s)}
    & \multicolumn{2}{c|}{\textbf{Obj. Pos. Error} (m)}
    & \multicolumn{2}{c}{\textbf{Obj. Quat. Error} (rad)} \\

    \cmidrule(lr){3-4} \cmidrule(lr){5-6} \cmidrule(lr){7-8} \cmidrule(lr){9-10} \cmidrule(lr){11-12}
     &  & Rot. & Flip. & Rot. & Flip. & Rot. & Flip. & Rot. & Flip. & Rot. & Flip. \\

    \midrule
    \multirow{6}{*}{\textbf{Fingertip}}

    & \textbf{SCSP*}
    & \textbf{0.97} & \textbf{0.85}
    & \textbf{10.09} & \textbf{15.68}
    & 0.023 & 0.023
    & \textbf{0.01} & \textbf{0.008}
    & \textbf{0.016} & \textbf{0.019} \\

    & \textbf{DyWA}
    & / & /
    & / & /
    & / & /
    & / & /
    & / & / \\

    & \textbf{Sampling}
    & 0.89 & 0.61
    & 18.53 & 18.56
    & 0.205 & 0.246
    & 0.015 & 0.016
    & 0.018 & 0.020 \\
    
    & \textbf{I-MPC}
    & 0.05 & 0.03
    & 19.62 & 24.89
    & 0.032 & 0.034
    & 0.016 & 0.017
    & 0.037 & 0.037 \\

    & \textbf{CF-MPC}
    & 0.07 & 0.05
    & 18.15 & 24.83
    & \textbf{0.006} & \textbf{0.006}
    & 0.018 & 0.012
    & 0.025 & 0.027 \\

    & \textbf{A-MPC}
    & 0.50 & 0.18
    & 12.21 & 19.72
    & 0.012  & 0.068
    & 0.015 & 0.016
    & 0.022 & 0.021 \\

    & \textbf{w/o RS}
    & 0.55 & 0.14
    & 19.42 & 23.78
    & 0.024 & 0.026
    & 0.019 & 0.019
    & 0.037 & 0.042 \\
    
    \midrule
    \multirow{7}{*}{\textbf{Franka}}

    & \textbf{SCSP*} 
    & \textbf{0.95} & \textbf{0.81}
    & \textbf{15.16} & \textbf{18.21}
    & 0.024 & \textbf{0.025}
    & \textbf{0.009} &\textbf{0.012}
    & \textbf{0.015} & \textbf{0.021} \\

    & \textbf{DyWA}
    & 0.86 & 0.69
    & 18.12 & 20.56
    & 0.031 & 0.031
    & 0.015 & 0.016
    & 0.019 & \textbf{0.021} \\
    
    & \textbf{Sampling}
    & 0.56 & 0.39
    & 18.97 & 19.68
    & 0.228 & 0.252
    & 0.014 & \textbf{0.012}
    & 0.016 & 0.022 \\

    & \textbf{I-MPC}
    & 0.01 & 0.00
    & 22.62 & /
    & 0.083 & 0.085
    & 0.018 & /
    & 0.036 & / \\

    & \textbf{CF-MPC}
    & 0.01 & 0.01
    & 21.18 & 22.56
    & 0.082 & 0.081
    & 0.016 & 0.018
    & 0.026 & 0.027 \\

    & \textbf{A-MPC}
    & 0.10 & 0.01
    & 18.32 & 23.72
    & \textbf{0.019}  & 0.028
    & 0.016 & 0.016
    & 0.021 & 0.020 \\

    & \textbf{w/o RS}
    & 0.36 & 0.18
    & 19.36 & 21.57
    & 0.023 & \textbf{0.025}
    & 0.019 & 0.018
    & 0.031 & 0.036 \\

    \bottomrule[1pt]
  \end{tabular}
  }
  \end{center}
  
\end{table*}

\begin{figure*}
 \centerline{\includegraphics[width=\textwidth]{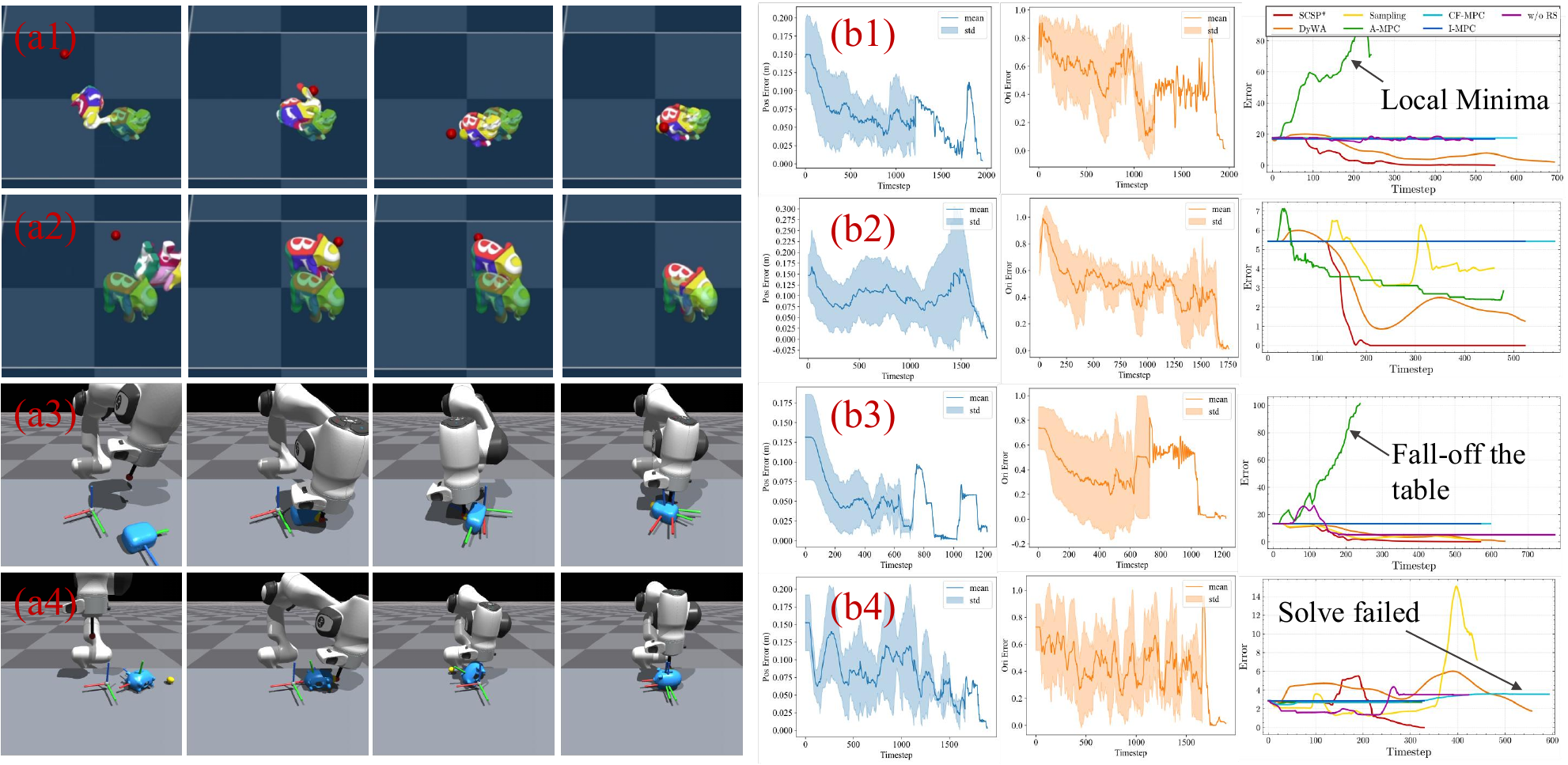}}
\caption{Nonprehensile manipulation experiments in simulation. (a1)-(a4): Snapshots of SCSP in the fingertip and robot experiments, where the goal pose is visualized either transparently or as a floating coordinate frame. (b1)-(b4): The left two panels show the mean and standard deviation curves of the position and orientation errors over steps across 10 SCSP trials. The right panel shows the pose error $\ell_{\text{pose}}$ of different methods during manipulation. The comparison methods either require longer execution time or fail due to local minima or solver failures.}\label{fig:simulation_data}
\end{figure*}

The experimental results are shown in Fig. \ref{fig:results_panda}, \ref{fig:simulation_data} and Table \ref{tab:all_data}. SCSP achieves a much higher success rate than other model-based algorithms and outperforms DyWA on all objects except the torus. Additionally, SCSP achieved the shortest execution time, demonstrating that SCSP can generate correct contact locations through CSO, thereby efficiently minimizing the pose cost of the object. SCSP also achieved the smallest convergence position error and quaternion error. I-MPC Furthermore, we observed that during experiments, DyWA tends to flip objects by pressing the edges of the top surface, whereas our method makes more comprehensive decisions regarding whether to lift the object by pushing the bottom or by pressing the top. Although many flipping tasks can be accomplished through pressing, for certain objects (such as elongated objects), the pressing approach requires applying substantial torque to the object and demands sufficiently high coefficients of friction. Based on our observations, the global contact selection mechanism in CSO enables SCSP to exhibit greater diversity in its manipulation behaviors compared with other methods, rather than converging to a single strategy. 

In addition, we conducted two ablation studies. CF-MPC and A-MPC can be regarded as ablations of the contact selection strategy. CF-MPC performs contact planning without contact selection, whereas A-MPC replaces the physical reasoning of CSO-based contact planning with a simple prior. Compared with SCSP, both methods show a substantial decrease in success rate, indicating that reasonable contact selection is crucial for  complex 6D nonprehensile manipulation tasks. The w/o RS variant serves as an ablation of the ranking strategy. We observed that, since w/o RS fully follows the high-frequency output of CSO, it frequently interrupts ongoing manipulation and switches contact points, resulting in a significant drop in success rate and an increase in execution time. This suggests that it is necessary to comprehensively evaluate the contact locations produced by CSO and to balance smooth execution against the global optimality of the selected contact location.

We further validated the impact of different numbers of sampled candidate points on performance. As shown in Table \ref{tab:cso_sampling}, when the number of sampled points reaches 70, nearly the highest success rate is achieved with relatively low SCSP computation time. As the number of sampled points continues to increase, the success rate plateaus and computational efficiency deteriorates. This indicates that increasing the number of sampled points does not necessarily improve performance. There exists an optimal convergence value. For nonprehensile manipulation tasks, 70 to 100 sampled points are sufficient to characterize the geometric features of object geometries.

\begin{table}[t]
\centering
\caption{Evaluation of the effect of the number of candidate points}
\label{tab:cso_sampling}
\begin{tabular}{lccccc}
\toprule
\textbf{Sampling numbers in CSO} & \textbf{5} & \textbf{10} & \textbf{70} & \textbf{120} & \textbf{500} \\
\midrule
\textbf{Avg. Success Rate}      & 0.25  & 0.59 & \textbf{0.88} & \textbf{0.88} & 0.85 \\
\textbf{SCSP Computation Time} (s)  & \textbf{0.012} & 0.017 & 0.025 & 0.032 & 0.089 \\
\bottomrule
\end{tabular}
\end{table}

\subsection{Real-world Experiments}

Finally, we validate the robustness of SCSP to in real-world experiments. Most model-based methods perform well in simulation, but in real-world experiments, they often rely on highly accurate perception systems, such as motion capture. The main challenge in replacing exact perception with purely vision-based perception is that errors in pose observation can perturb contact estimation, leading to incorrect inference of contact states and consequently inaccurate contact dynamics. To the best of our knowledge, our method is the first model-based approach that can achieve purely vision-based, complex geometries, model mismatch contact-rich manipulation.

\begin{figure}[!t]
 \centerline{\includegraphics[width=\columnwidth]{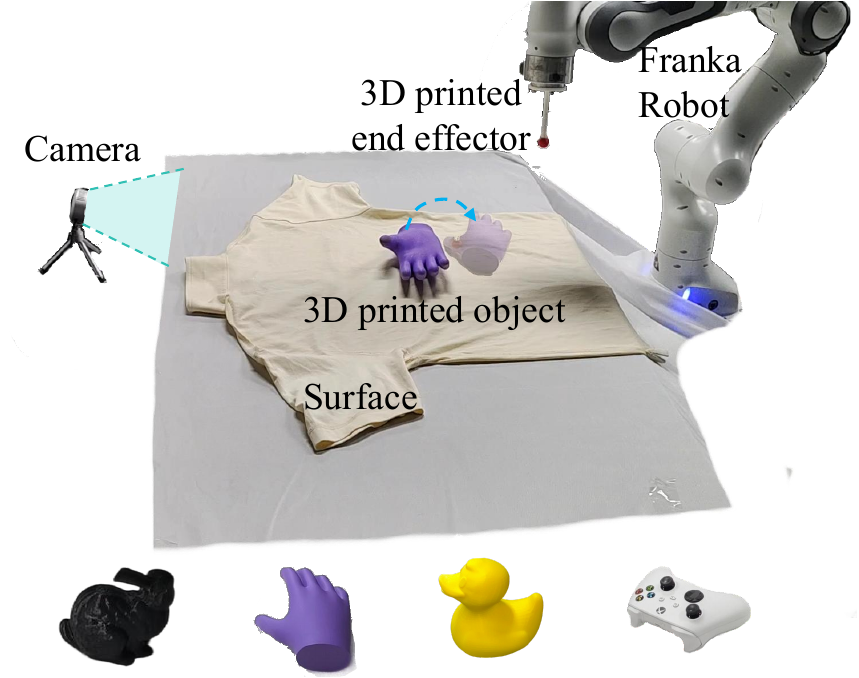}}
\caption{Hardware setup of the real-world experiments.}
\label{fig:illustrative_example_1}
\end{figure}
\subsubsection{Experimental Setup}
We conduct experiments using a Franka Emika Panda equipped with a 3D-printed end-effector, and arbitrarily select 4 irregularly shaped objects: Duck, Bunny, Hand, and Xbox Controller, as shown in Fig. \ref{fig:illustrative_example_1}. We evaluate on-ground rotation on Duck and Bunny, and on-ground flipping on Bunny, Hand, and Xbox Controller. The object CAD models are obtained from the ContactDB dataset and online repositories. To validate the robustness of our method under inaccurate dynamics, we provide only rough estimates of the object’s physical parameters such as mass and friction coefficients. Moreover, instead of relying on precise motion capture for exact perception, we perform state estimation by \texttt{FoundationPose} \cite{wen2024foundationpose} using a single-view camera (Realsense D435i). In real-world experiments, precise object–environment contact locations are not directly observable. To address this, we approximately reconstruct the task environment in \texttt{MuJoCo} for contact estimation, as shown in the Fig. \ref{fig:illustrative_example_i}. We perform a simple calibration procedure to determine the object’s scale parameter and positional offsets in the simulator to eliminate contact estimation misclassification. Specifically, we use Franka’s kinesthetic teaching mode to move the end-effector and make contact with several randomly selected locations on the object, and then tune the object’s scale and the offsets $[\Delta x, \Delta y, \Delta z]$ such that the simulator can approximately predict whether the end-effector is in contact with the object. The goal pose range, success conditions, and maximum execution steps for the task are identical to those in the simulation experiments. Sampling and DyWA use the same settings as in the previous section, and the outputs of DyWA are executed using an impedance controller.
\begin{figure}[!t]
 \centerline{\includegraphics[width=\columnwidth]{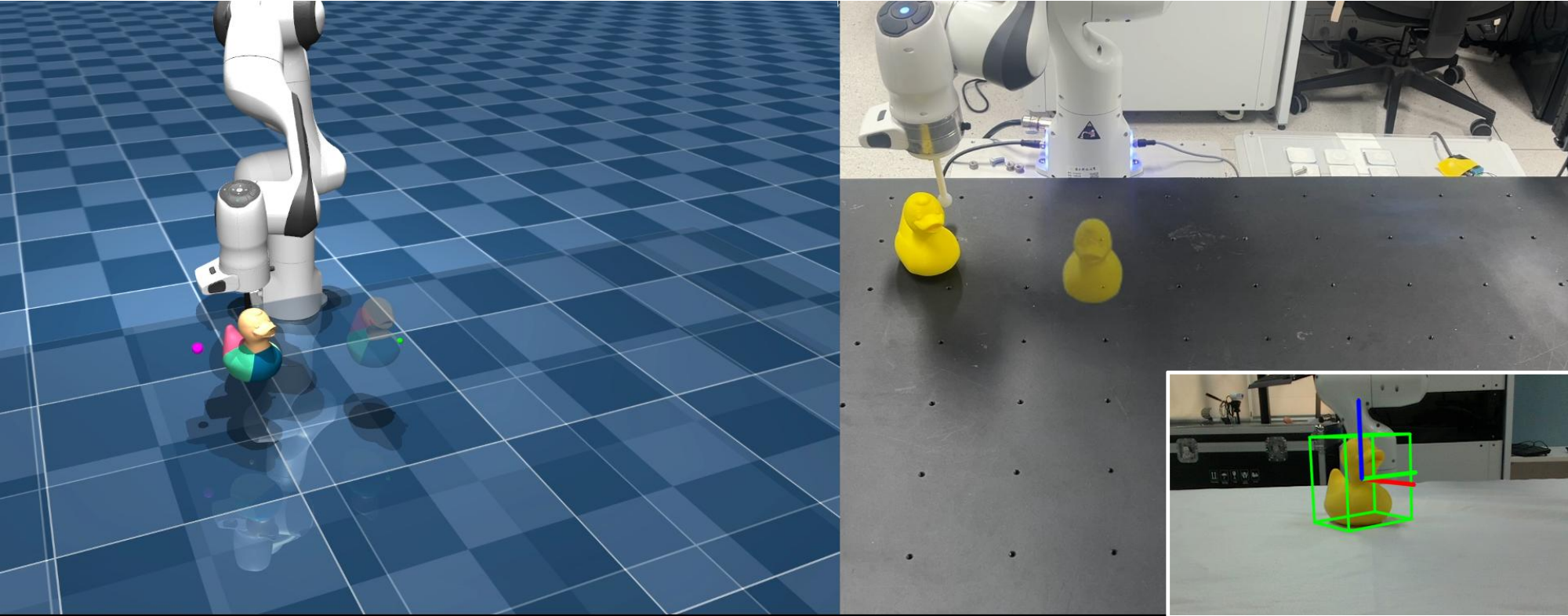}}
\caption{Perception method in the real-world experimental setting. Object pose is detected using \texttt{FoundationPose} and filtered via a Kalman filter. Contact estimation is performed using a \texttt{MuJoCo} implementation that approximates the experimental environment.}
\label{fig:illustrative_example_i}
\end{figure}

\begin{figure*}
 \centerline{\includegraphics[width=\textwidth]{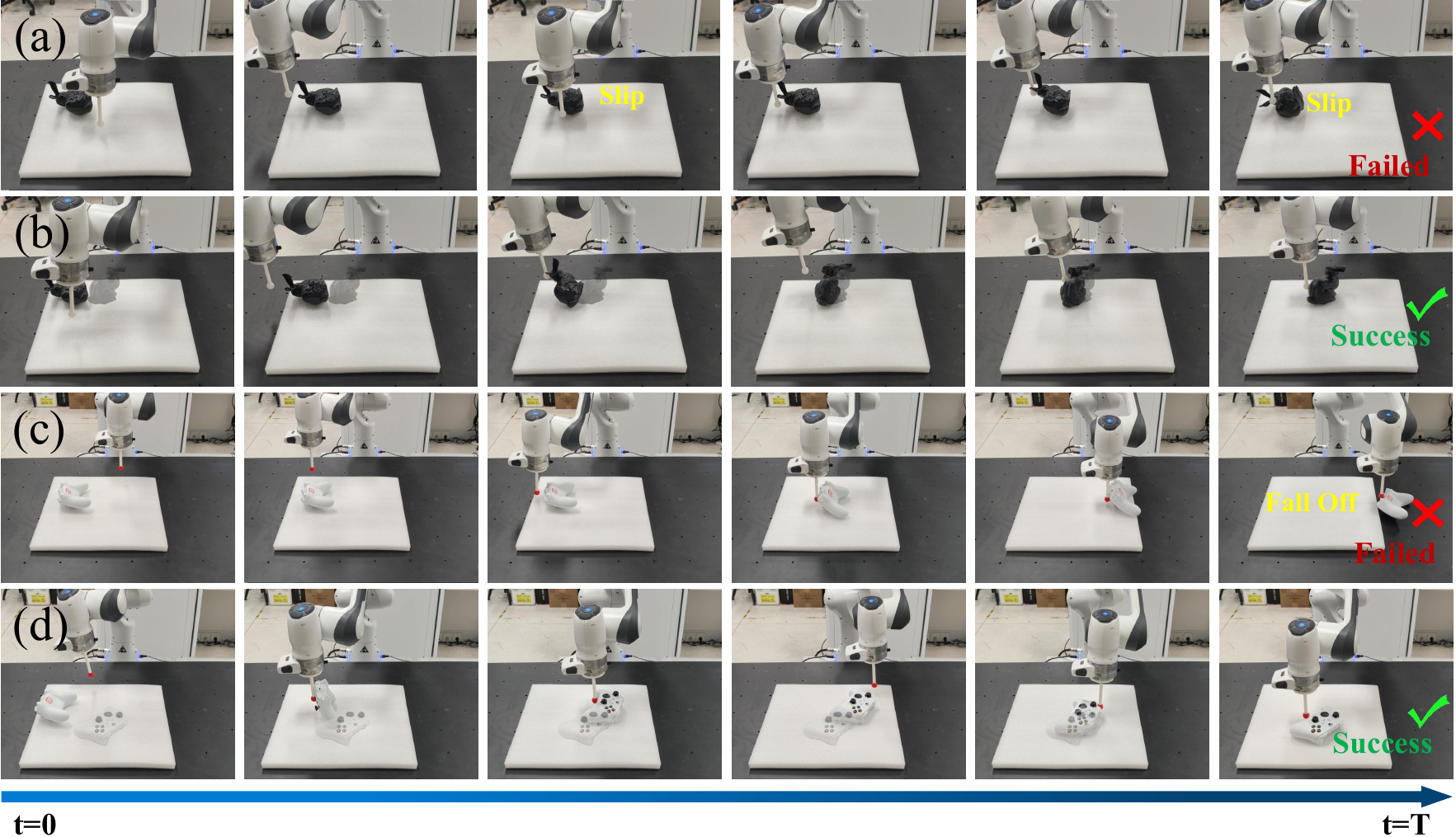}}
\caption{Snapshots of nonprehensile manipulation in real-world experiments. (a) \textbf{DyWA} failed to flip the bunny. (b) \textbf{SCSP} succeeds in flipping the bunny. (c) \textbf{Sampling} failed by pushing the Xbox controller to fall off the foam padding. (d) \textbf{SCSP} succeed in flipping the Xbox controller.}\label{fig:results_real_world_flip}
\end{figure*}
The OSC controller is implemented in \textsf{c++}. Unlike in simulation experiments, control in real-world deployment is subject to disturbances, leading to control lag and steady-state errors. To address this, we designed an adaptive controller for precise control, which we refer to as the Incremental OSC. We introduce an intermediate position vector $\boldsymbol{p}_{\text{inc}} = \boldsymbol{p}_{\text{inc}} + \boldsymbol{K}_{\text{inc}}\boldsymbol{u}$,
and replace $\boldsymbol{e}_{\text{pos}}$  with $\boldsymbol{e}_{\text{pos}}=\boldsymbol{p}_{\text{inc}}-\boldsymbol{p}_{\text{ee}}$. The steady-state error is adaptively eliminated under the high-frequency self-integrating effect. In the experiment, we set
$\boldsymbol{K}_{\text{inc}}=0.01 \boldsymbol{I}_3$, $\boldsymbol{K}_x = \text{diag}(500,500,500,10,10,10)$ and $\boldsymbol{D}_x = \sqrt{2\boldsymbol{K}_x}$.

\subsubsection{Results and Discussion}

\begin{table*}[t]
\centering
\caption{Success rates (\%) for on-ground rotation and flipping on different surfaces and objects.}
\label{tab:real_world_success}
\setlength{\tabcolsep}{5pt}
\begin{tabular*}{\textwidth}{@{\extracolsep{\fill}} lcccccccccc}
\toprule
& \multicolumn{4}{c}{\textbf{On-ground rotation}}
& \multicolumn{6}{c}{\textbf{On-ground flipping}} \\
\cmidrule(lr){2-5} \cmidrule(lr){6-11}
\textbf{Method}
& \multicolumn{2}{c}{Duck}
& \multicolumn{2}{c}{Bunny}
& \multicolumn{2}{c}{Bunny}
& \multicolumn{2}{c}{Hand}
& \multicolumn{2}{c}{Xbox} \\
\cmidrule(lr){2-3} \cmidrule(lr){4-5}
\cmidrule(lr){6-7} \cmidrule(lr){8-9} \cmidrule(lr){10-11}
& Smooth & Rough
& Smooth & Rough
& Smooth & Rough
& Smooth & Rough
& Smooth & Rough \\
\midrule
Sampling
& 2 / 5 & 1 / 5
& 1 / 5 & 0 / 5
& 0 / 5 & 0 / 5
& 0 / 5 & 0 / 5
& 0 / 5 & 0 / 5 \\
DyWA
& 2 / 5 & 2 / 5
& 2 / 5 & 1 / 5
& 0 / 5 & 0 / 5
& 0 / 5 & 0 / 5
& 0 / 5 & 0 / 5 \\
SCSP*
& \textbf{5 / 5} & \textbf{4 / 5}
& \textbf{5 / 5} & \textbf{5 / 5}
& \textbf{2 / 5} & \textbf{3 / 5}
& \textbf{3 / 5} & \textbf{4 / 5}
& \textbf{4 / 5} & \textbf{5 / 5}\\
\bottomrule
\end{tabular*}
\end{table*}

We first evaluate the success rate of SCSP across different objects and tabletop materials to assess its robustness under model inaccuracies. The experimental results are shown in Fig. \ref{fig:results_real_world_flip}, \ref{fig:real_world_continous} and Table \ref{tab:real_world_success}. They indicate that our method maintains strong performance even under perception noise and model mismatch.

We conducted each experiment on two different tabletop materials: smooth and rough as in Fig. \ref{fig:real_world_disturbance}. For the on-ground rotation experiments, the smooth surface denotes a tabletop with low friction, whereas the rough surface denotes cotton clothes. For the on-ground flipping experiments, which require greater friction, we use foam padding as the rough surface and cotton cloth as the smooth surface. Note that our method does not measure friction coefficients, object mass, or other physical parameters, but instead uses coarse estimates. Additionally, biases in contact estimation and noise in object pose estimation can interfere with the decision-making of SCSP. The experiment results suggest that SCSP can effectively cope with model inaccuracies and disturbances, which is typically challenging for model-based methods. 

As shown in Table \ref{tab:real_world_success}, SCSP achieved a 100\% success rate in the on-ground rotation experiments with complex object geometries and also showed strong performance in the flipping tasks, whereas the comparison methods failed in nearly all flipping experiments. We observe that the failures of DyWA mainly arose from contact slippage and violations of the Franka acceleration limits during pressing, while the failures of Sampling were primarily caused by the selection of unstable contact locations. Due to the complex geometry and relatively smooth material of the 3D printed objects, the fingertip sometimes fails to establish stable contact with the object, resulting in relative sliding that causes the execution of the contact plan planned by SCSP to fail. Additionally, surface friction also affects the success rate of flipping. As shown in the results, SCSP achieves a notably higher success rate on the rough surface compared with the smooth surface.

We further conduct continuous manipulation experiments, as shown in Fig. \ref{fig:real_world_continous}. During manipulation, we repeatedly perturb the object pose to evaluate the real-time performance, closed-loop planning capability, and robustness of SCSP under accumulated perception errors. Under repeated perturbations, both the pose estimated from \texttt{FoundationPose} and the contact estimation become biased. However, SCSP exhibits strong replanning capability and consistently completes the continuous manipulation tasks under disturbances.

We combine the preceding experimental results to analyze the source of SCSP’s robustness to model mismatch. We conclude that: 1) the ablation study reveals that optimality of the contact prior is not necessary. Although the prior is important for contact planning, even a coarse contact location selection as a prior can still bring substantial performance improvement. This observation also provides insight into the role of the world model: how to properly connect the prior with the planning framework seems to be more important than the optimality and accuracy of the prior itself. The designs of the CSO and the ranking strategy are both developed with this perspective. 2) CPO has a strong closed-loop nature. The real-time performance and robustness of CIMPC methods have been thoroughly validated in previous work. CPO is a prior-guided contact planning framework built upon the CIMPC formulation. It inherits the robustness of CIMPC while improving manipulation diversity. Therefore, the SCSP framework is able to achieve robust performance.

\begin{table}[t]
\centering
\caption{Computation time of SCSP}
\label{tab:metrics}
\begin{tabular*}{0.8\columnwidth}{@{\extracolsep{\fill}} l c}
\toprule
\textbf{Metric} & \textbf{Value} \\
\midrule
\multicolumn{2}{c}{\textbf{CSO Computation time}} \\
\midrule
Sampling cost time (s)           & $4.82 \pm 3.85$ \\
Exchange cost time (ms)          & $0.002 \pm 0.01$ \\
MIQP cost time (ms)              & $0.23 \pm 0.01$ \\
CSO solving time (ms)            & $16.21 \pm 0.58$ \\
\midrule
\multicolumn{2}{c}{\textbf{CPO Computation time}} \\
\midrule
RS cost time (s)                 & $0.01 \pm 0.00$ \\
Average CPO solving time (ms)  & $8.68 \pm 0.31$ \\
\bottomrule
\end{tabular*}
\end{table}

\begin{figure*}
 \centerline{\includegraphics[width=\textwidth]{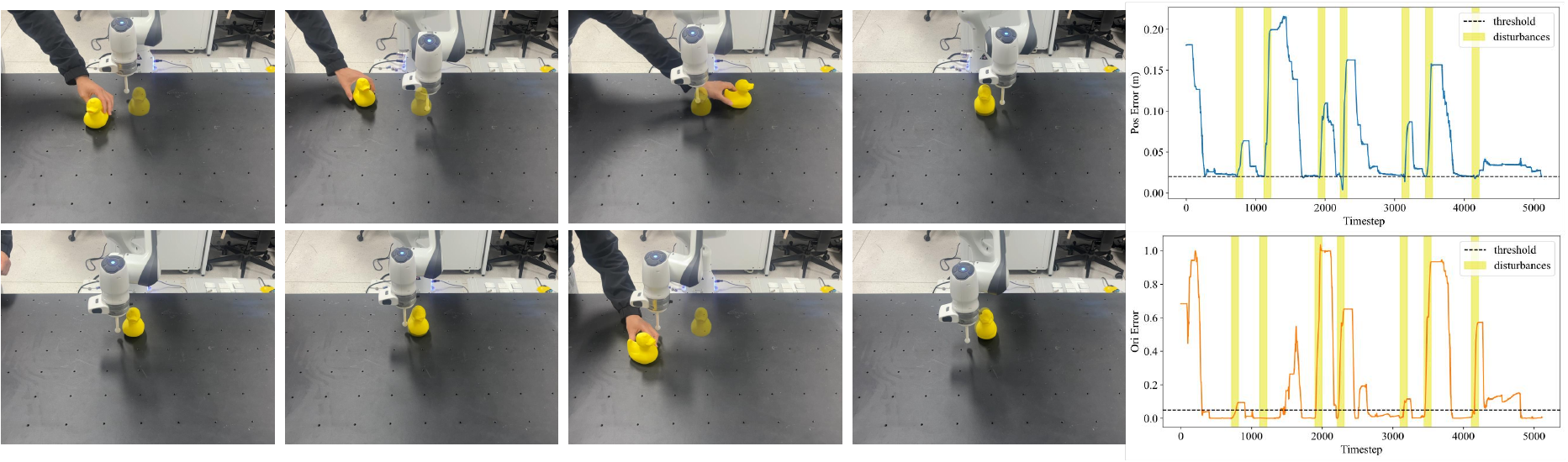}}
\caption{Snapshots of manipulation under external disturbances. The object’s pose is continuously perturbed during manipulation to demonstrate SCSP’s real-time robustness under such disturbances. The panels on the right shows the temporal evolution of position and orientation errors.}\label{fig:real_world_continous}
\end{figure*}

\begin{figure}
 \centerline{\includegraphics[width=\columnwidth]{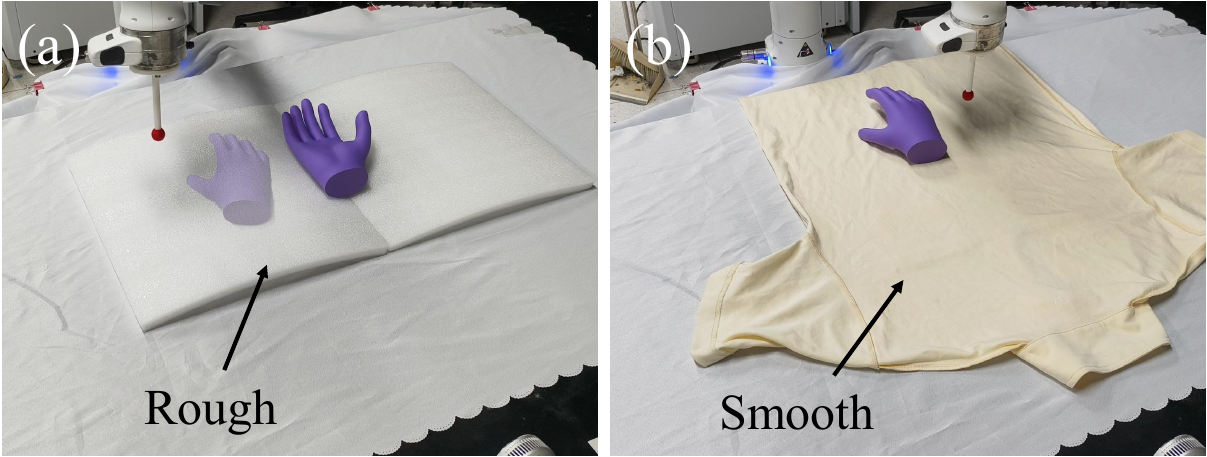}}
\caption{Robustness validation of SCSP on surfaces with different friction. (a) Foam. (b) Cotton clothes.}\label{fig:real_world_disturbance}
\end{figure}

%% file: generalize.tex
\section{Generalization to Multiple Tasks}\label{sec:generalize}

We further implement SCSP in different scenarios to evaluate its generalization. Using multi-contact tasks and long-horizon tasks as examples, we demonstrate that SCSP can perform general physical reasoning and execute more complex manipulation tasks on more underactuated objects.

\subsection{Case Study I: Multi-Contact Task}

SCSP can be easily extended to multi-contact tasks. We illustrate the generalization of SCSP with following example:
\subsubsection{Soccer Dribbling}
We use a soccer dribbling task to demonstrate the capability of SCSP in dynamic scenarios. In this task, SCSP plans contact with the ball using one of the two feet and drives it toward the goal through multi-contact manipulation. The main challenges lie in the highly dynamic behavior of the ball, which makes stable and sustained contact difficult, and in coordinating the two feet effectively. We read the ball position $\boldsymbol{p}_k$ and estimate its velocity $\boldsymbol{v}_k$ from the simulation using a Kalman filter. The 3D model of the ball is then input to the SCSP to compute the actions $\boldsymbol{u} \in \mathbb{R}^{2\times 4}$ for both feet, where the control vector for each foot is $\left [ x,y,z,\theta_{\mathrm{yaw}} \right ] $. We define activation triggers separately for the left and right feet: using the midpoint between both feet as the center, the space is divided into left and right half-spaces. When $\boldsymbol{x}_k$ appears in a half-space, the corresponding left or right foot is activated. Additionally, workspace constraints (represented as semi-transparent cubes) are imposed on each foot to prevent singular gait patterns. The number of CSO sampling points is set to $n_s=70$, and points close to the ground $\left \{ p_i\in\mathcal{I}_{\text{obj}} \ | \ z_i < 0.02\mathrm{m}  \right \}$ are masked out. We randomly selected 10 goal positions for the soccer ball, and SCSP continuously controlled the robot’s feet to dribble the ball to each target, achieving a 100\% success rate. The experimental results are shown in Fig. \ref{fig:soccer}. This experiment serves as a toy example, simulating the approaching and contact phases in humanoid soccer dribbling. In the future, we plan to extend SCSP to more complex motion planners to realize a more complete demonstration \cite{yin2025visualmimic, kong2026learning}.
\begin{figure}[!t]
 \centerline{\includegraphics[width=\columnwidth]{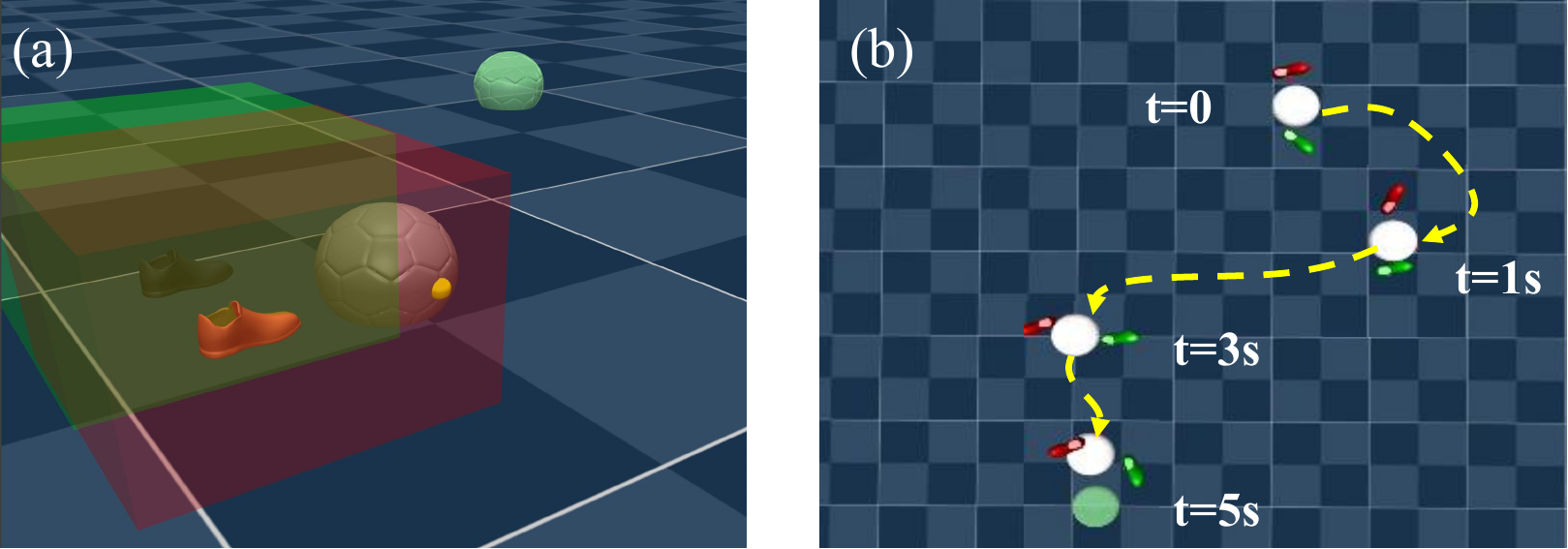}}
\caption{Soccer dribbling task. (a) Experiment setup. The task is to kick the soccer to the target (translucent). Red and green cubes show footworkspaces, and the yellow point indicates CSO-planned contact locations. (b) Gribbling trajectory.}
\label{fig:soccer}
\end{figure}
\subsubsection{Tilted Pushing} 

SCSP can be extended to manipulation tasks in uneven environments, such as tilted pushing. Stable object poses on a slope are inherently fragile, requiring accurate contact location selection and online replanning. We compare open-loop methods, which plan a contact sequence offline, with SCSP, which replans online during execution. As the baseline, we use STOCS-3D to compute an offline sequence of contact locations and contact forces, and execute this sequence using our CPO. Experiments are conducted in \texttt{IsaacGym} with a 15$^\circ$ slope and a friction coefficient of 0.1. For SCSP, we set $w_{\text{stab}}=0$ and $\hat{\lambda}^n_{\text{max}}=0.5$, while keeping all other parameters the same as in Section~\ref{sec:robot_manipulation}. We selected the teapot and duck (used in previous experiments) and randomly set 10 stable poses on the slope as goal poses, and pushed and rotated the objects from below to reach these goals. The results in Table~\ref{tab:generalization_tasks_result} and Fig.~\ref{fig:tilted_push} show that the real-time replanning capability of SCSP provides substantially better robustness.

\begin{figure}[!t]
 \centerline{\includegraphics[width=\columnwidth]{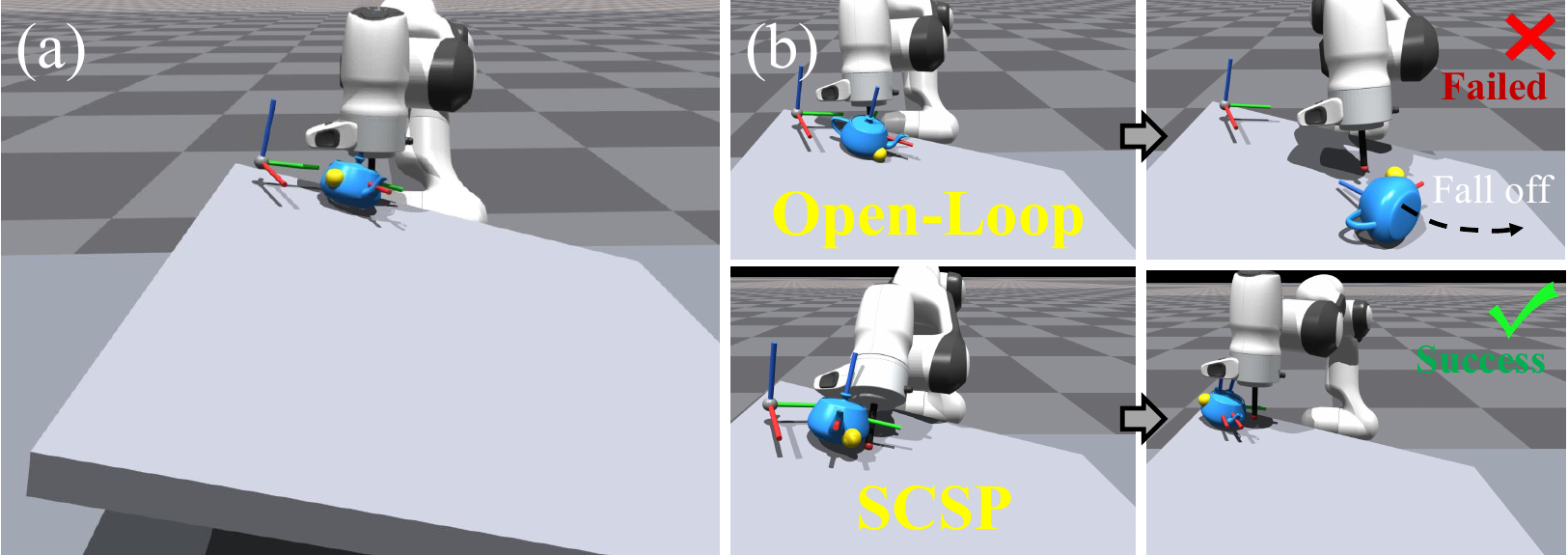}}
\caption{Experiment setup. The task is to push an object on a surface tilted at 15$^\circ$. (a) The open-loop method fails when execution errors occur, causing the teapot to fall off the tilted surface. (b) SCSP can replan online and thus successfully complete the dynamic task.}\label{fig:tilted_push}
\end{figure}
\subsubsection{Grasping \& Dexterous Manipulation} 
SCSP can also be applied to predict grasping poses, such as bimanual grasping and multi-finger grasping. Following \cite{lin2026bidexgrasp}, we use the Grasp Wrench Boundary (GWB) \cite{borst2004grasp} to select grasping regions and identify the top-K pairs of potential valid regions for grasping. Subsequently, sampling of the candidate set $\mathcal{I}_{\text{obj}}$ is performed within these regions. The objective of the CSO inner loop optimization is then transformed to force closure:
\begin{equation}\label{equ:force_closure}
    \begin{aligned}
\min_{\boldsymbol{\hat{\lambda}}_r \in \mathcal{K}}
\quad &
\sum_{j=1}^{N_d} \left\| \xi \boldsymbol{w}^{(j)} - \boldsymbol{J}_o \boldsymbol{\hat{\lambda}}_r^{(j)} \right\|_2^2
\;+\;
w_{\lambda} \sum_{j=1}^{n_d} \left\| \boldsymbol{\hat{\lambda}}_r^{(j)} \right\|_2^2 \\
\text{s.t.}\quad
& \sum_{i=1}^{K} \hat{\lambda}_r^{(j),n} \ge \zeta, \ 
\forall j=1,\dots,n_d .
\end{aligned}
\end{equation}
where $n_c$ is the number of robot contact points, $n_d$ denotes the number of sampled disturbance wrenches, $\boldsymbol{w}^{(j)} \in \mathbb{R}^6$ represents the $j$-th disturbance wrench. $\boldsymbol{\lambda}_r^{(j)}$ denotes the robot contact force corresponding to disturbance $\boldsymbol{w}^{(j)}$. $\xi$ is the disturbance scaling factor, $\zeta$ the minimum total normal force required to maintain an active grasp and $w_{\lambda}=0.1$. During each rollout, the CSO in SCSP solves \eqref{equ:force_closure} for the nearest points $\boldsymbol{p}_{\text{near}}^i$ to the robot end-effector or Allegro fingers additional to the top 5 pairs of potential grasping pairs, thereby searching for grasping poses online. We use an open-loop baseline that computes the optimal grasping locations offline using CSO and executes them online with CPO for comparison. For bi-manual and multi-finger grasping, we selected five objects each, placed them on a pedestal (black), randomly perturbed their yaw angles $\theta_{\text{yaw}}$ and scaled their sizes by a random factor $\varsigma \in \left \{ 0.5,1,1.5,2 \right \} $, conducting 20 trials for each tasks. The results are shown in Table~\ref{tab:generalization_tasks_result} and Fig.~\ref{fig:grasping}. SCSP significantly outperforms the baseline in success rate, but note that it fails on objects with complex geometry due to \texttt{MuJoCo}'s automatic convexification of collision mesh.
\begin{figure}[!t]
 \centerline{\includegraphics[width=\columnwidth]{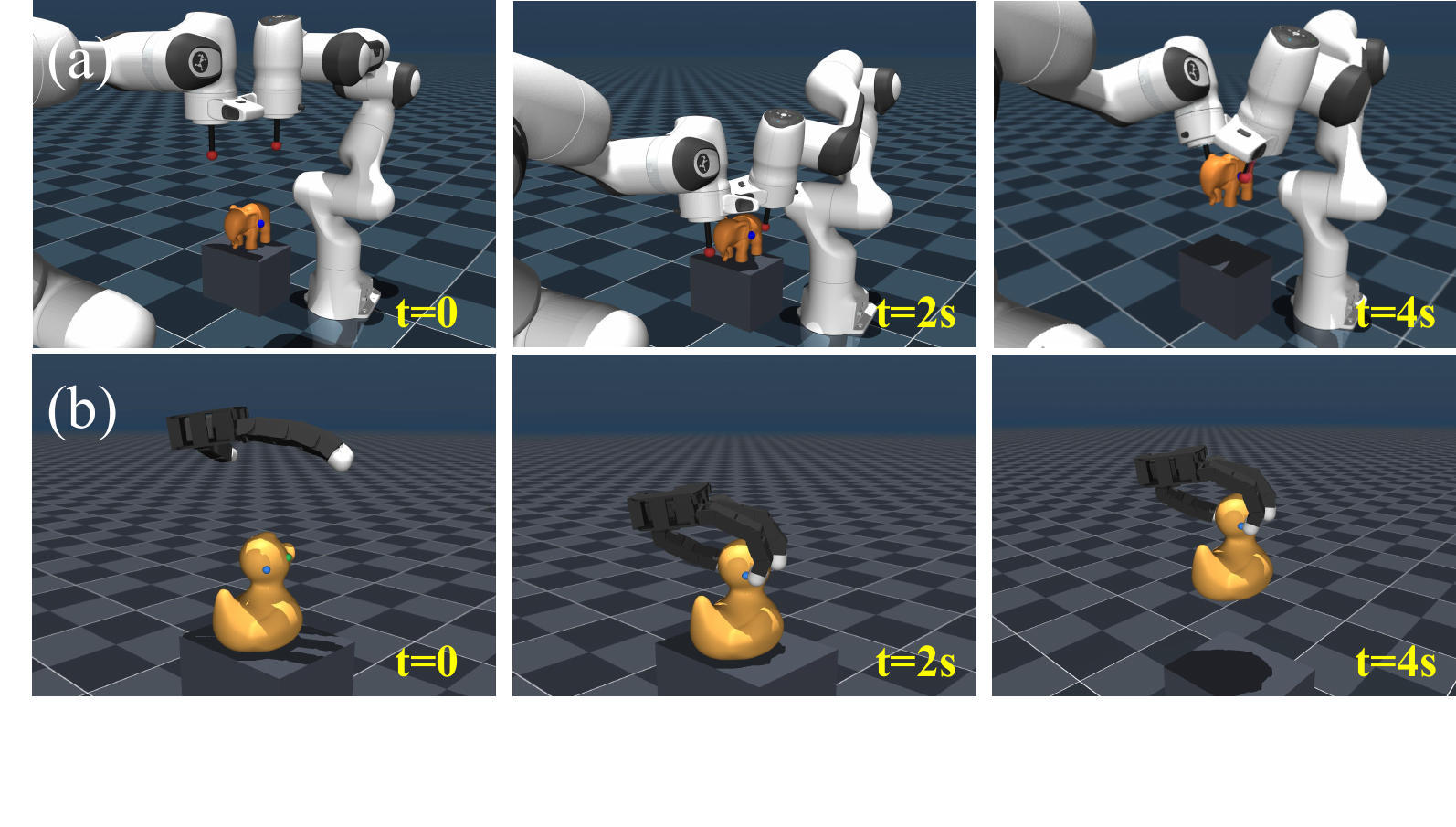}}
\caption{Snapshots of the multi-contact grasping tasks. (a) Bi-manual grasping of an elephant with the size scaling factor $\varsigma=1$. (b) Multi-finger grasping of a duck with $\varsigma=1.5$.}
\label{fig:grasping}
\end{figure}
\subsection{Case Study II: Long-Horizon Task}
We use this experiment to demonstrate that when SCSP has the capability to actively select contact points, the contact planning algorithm exhibits physical inference abilities, enabling it to accomplish long-horizon tasks. The tasks are to pull the drawer to a target displacement and to open the microwave door to a target angle. In both tasks, the object is semantically segmented into different parts, such as handles, while non-manipulable parts are excluded. For the drawer, we directly compute the final pulled-out pose as the goal pose. For the microwave, which involves rotational motion, we set the goal pose of each rollout as the \texttt{SLERP} interpolation between the current state and the desired door state. The parameter and cost settings of SCSP are the same as those in Section~\ref{sec:robot_manipulation}, with $w_{\text{stab}}=0$. The solver for CPO is an MPPI-based motion planner~\cite{bhardwaj2022storm}. As a baseline, we use an open-loop method that plans the contact sequence offline using CSO and executes it online with CPO, in order to demonstrate the effectiveness of online contact selection. We randomly perturbed the distance and yaw angle between the drawer or microwave and the Franka base by $\Delta d \in \left [ -0.1, 0.1 \right ] \mathrm{m}$ and $\Delta \theta_{\text{yaw}} \in \left [ -\frac{\pi}{4} , \frac{\pi}{4}  \right ]$ and conducted 10 trials for each object. The experimental results are shown in Fig.~\ref{fig:long_horizon}, Table~\ref{tab:generalization_tasks_result} and the \href{https://sites.google.com/view/scsp-robot}{website} videos. Although the success rates are comparable, the open-loop method cannot handle certain execution errors, such as contact slippage or the robot approaching singular configurations.
\begin{figure}[!t]
 \centerline{\includegraphics[width=\columnwidth]{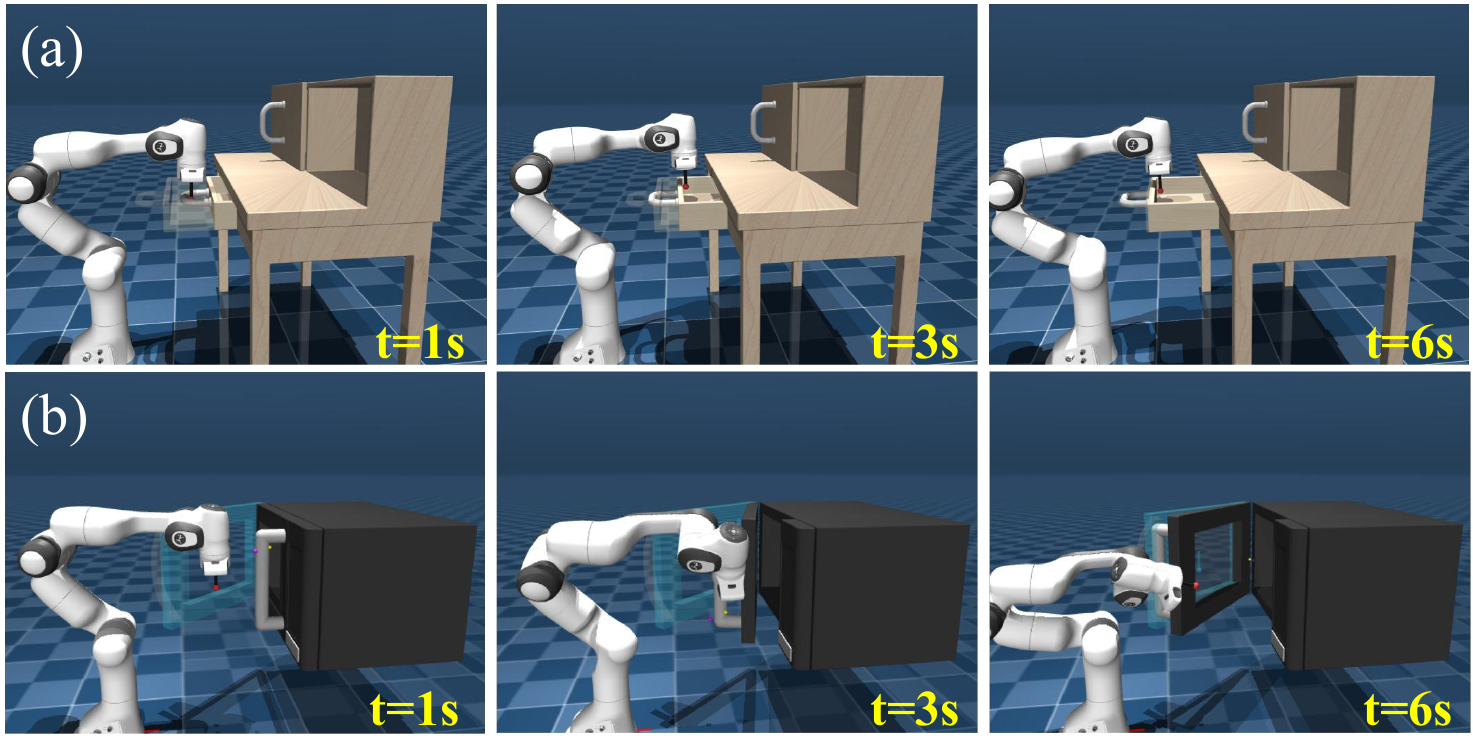}}
\caption{Snapshots of long-horizon tasks. Given only the target pose of the object, SCSP performs online reasoning and selects appropriate contact locations to accomplish the task. (a) Opening a drawer. (b) Opening the microwave door.}
\label{fig:long_horizon}
\end{figure}

\begin{table*}[t]
\centering
\caption{Success rates (\%) across different tasks.}
\label{tab:generalization_tasks_result}
\setlength{\tabcolsep}{5pt}
\begin{tabular*}{\textwidth}{@{\extracolsep{\fill}} lccccccccc}
\toprule
\textbf{Metric}
& \textbf{Soccer Dribbling}
& \multicolumn{2}{c}{\textbf{Tilted Pushing}}
& \multicolumn{2}{c}{\textbf{Bi-Manual Grasping}}
& \multicolumn{2}{c}{\textbf{Multi-finger Grasping}}
& \multicolumn{2}{c}{\textbf{Long-horizon Tasks}} \\
\cmidrule(lr){2-2} \cmidrule(lr){3-4} \cmidrule(lr){5-6} \cmidrule(lr){7-8} \cmidrule(lr){9-10}
& SCSP
& Open-Loop & SCSP
& Open-Loop & SCSP
& Open-Loop & SCSP
& Open-Loop & SCSP \\
\midrule
\textbf{Success Rate}
& \textbf{10 / 10}
& 0 / 20 & \textbf{18 / 20}
& 0 / 20 & \textbf{15 / 20}
& 0 / 20 & \textbf{16 / 20}
& 15 / 20 & \textbf{20 / 20} \\
\bottomrule
\end{tabular*}
\end{table*}

%% file: discussion.tex
\section{Discussion \& Limitations}
\label{sec:discussion}

Through the above theoretical and experimental analyses, we demonstrate the importance of contact selection and the validity of SCSP. However, SCSP also exhibits limitations.

The broader significance of CSO lies in providing a world model that supplies priors for contact locations. Our CSO is implemented using mesh discretization and a lightweight contact model. Although it is computationally efficient, it lacks accuracy in complex contact scenarios and cannot perform sequential state prediction. Furthermore, CSO is restricted to modeling contacts as point contacts, preventing more realistic simulations. Consequently, SCSP faces limitations when applied to contact planning involving robots with complex end-effector geometries or tasks involving deformable objects. Replacing the CSO module with a more expressive world model, such as \cite{huang2026pointworld}, may lead to improved performance.

CPO is an online trajectory optimization method for generating manipulation trajectories in both contact-free and contact-rich stages. Existing CPO solvers are based on MPC, which limits their applicability to more complex tasks such as loco-manipulation. Furthermore, CPO struggles to maintain continuous contact along object surfaces while switching contact locations, which could enable smoother manipulation in certain tasks. The current rule-based ranking strategy could be replaced with a more expressive latent representation, and the trajectory optimization module could be substituted with a generative motion planner such as \cite{xu2026interprior} to support more complex motion generation.

We further discuss insights from the SCSP architecture. The experimental results demonstrate that SCSP can perform online reasoning in complex contact-rich manipulation tasks and exhibits robustness under perceptual noise and inaccurate dynamics surpassing traditional methods. The contact priors produced by CSO enhance the diversity of manipulation strategies, while CPO leverages sub-optimal solutions as candidate contact locations to mitigate instability caused by inaccuracies in CSO outputs due to model errors. A similar observation holds for the world model: when validation tasks are clearly out of distribution, guidance from the world model helps the planner escape local optima. Notably, the accuracy of the world model is less critical than how it interfaces with the underlying planner, which has a greater impact on task success. We believe that replacing CSO with a more general world model and CPO with a sufficiently generative neural motion planner could yield a unified multi-task framework for a wide range of contact-rich manipulation tasks. These directions will be explored in future work.

%% file: conclusion.tex
\section{Conclusion} 
\label{sec:conclusion}

In this work, we propose a novel cascaded optimization framework named SCSP for contact-rich manipulation. Most existing contact planning methods lack the ability to actively select contacts and autonomously generate diverse manipulation trajectories, which limits them to relatively simple tasks. To address this issue, we design a cascaded optimization framework consisting of CSO and CPO to perform simultaneous contact selection and contact planning. CSO enables online optimization of the optimal contact points for objects with arbitrarily complex geometries through sampling and a surrogate contact model. CPO integrates the contact priors produced by CSO using a ranking strategy to evaluate the priors, and online generates the approaching, contact, and contact-switching trajectories. Extensive experimental results demonstrate the superiority of SCSP and the soundness of its design. We further validate the generalization of our method across numerous complex tasks, showing that it can serve as a general framework for autonomous contact-rich manipulation.

%% file: appendix.tex
\subsection{Proof of Closest Approximation}\label{proof:closest_approximation}
Let $\boldsymbol{W}_{ii} \in \mathbb{R}^{n \times n}$ be the original Delassus matrix. We seek a diagonal matrix $D = \operatorname{diag}(d_1, \dots, d_n)$ 
that minimizes the Frobenius norm of the approximation error:
\begin{equation}
    \boldsymbol{D}^* = \arg \min_{\substack{\boldsymbol{D}}} 
    \left\| \boldsymbol{W}_{ii} - \boldsymbol{D} \right\|_{\mathrm{F}}^2.
\end{equation}
Expanding the Frobenius norm, we have
\begin{equation}
    \| \boldsymbol{W}_{ii} - \boldsymbol{D} \|_\text{F}^2 = \sum_{j \neq k} |(\boldsymbol{W}_{ii})_{jk}|^2 + \sum_{j=1}^{n} |(\boldsymbol{W}_{ii})_{jj} - d_j|^2.
\end{equation}
The off-diagonal terms are not affected by choosing $D$, and so the minimization reduces to independent scalar problems for each diagonal entry:
\begin{equation}
d_j^* = \arg \min_{d_j} |(\boldsymbol{W}_{ii})_{jj} - d_j|^2, \quad j = 1, \dots, n.
\end{equation}
It is immediate that the minimizer is $d_j^* = (\boldsymbol{W}_{ii})_{jj}$. 
Hence \eqref{equ:max_decomposition} is the closest diagonal approximation in the Frobenius norm with a small regularization for  numerical stability.
\subsection{Proof of Lemma 3}
\label{proof:lemma3}
The SCM simplification consists of two steps: (i) decomposing the Delassus matrix $W_{\text{env}}$ into $3\times 3$ blocks for each contact, and (ii) taking the diagonal elements (block-diagonal approximation) to compute the closed-form solution. By standard perturbation theory for LCP, the solution is Lipschitz continuous with respect to perturbations in the matrix:
\begin{equation}
\left\| \boldsymbol{\lambda}_{\text{env}} - \hat{\boldsymbol{\lambda}}_{\text{env}} \right\|_2 
\le \sigma \, \left\| \boldsymbol{W}_{\text{env}} - \boldsymbol{D} \right\|_2,
\end{equation}
where $\sigma$ is inversely proportional to the minimal eigenvalue of $W_{\text{env}}$.  

The velocity error induced by this approximation is then
\begin{multline}
\|\boldsymbol{v}^+ - \boldsymbol{\hat{v}}^+\|_2 
= \Bigg\| \boldsymbol{M}_o^{-1} \sum_i \boldsymbol{J}_i^\top 
\big(\boldsymbol{\lambda}_{{\text{env}},i} - \boldsymbol{\hat{\lambda}}_{{\text{env}},i}\big) \Bigg\|_2 \\
\le \|\boldsymbol{M}_o^{-1}\|_2 \sum_i \|\boldsymbol{J}_i^\top\|_2 \, 
\|\boldsymbol{\lambda}_{{\text{env}},i} - \boldsymbol{\hat{\lambda}}_{{\text{env}},i}\|_2.
\end{multline}
Finally, the state update error satisfies
\begin{equation}
\| \boldsymbol{x}_{k+1} - \boldsymbol{\hat{x}}_{k+1} \|_2 = h \|\boldsymbol{v}^+ - \boldsymbol{\hat{v}}^+\|_2 \le \boldsymbol{C} \, \| \boldsymbol{W}_{\text{env}} - \boldsymbol{D} \|_2,
\end{equation}
where $\boldsymbol{C} = h \, \|\boldsymbol{M}_o^{-1}\|_2 \sum_i \|\boldsymbol{J}_i^\top\|_2 \, \sigma$. Since $\boldsymbol{D}$ is the closest approximation of $\boldsymbol{W}_{\text{env}}$, $\| \boldsymbol{W}_{\text{env}} - \boldsymbol{D} \|_2$ is bounded, thus the SCM provides a bounded-error approximation to the true complementarity solution.
\subsection{Proof of Lemma 4}
\label{proof:lemma4}
Consider the lift-and-place cost 
$\ell_{\text{lp}} = (1-\gamma) \, \ell_{\text{lift}} + \gamma\, \ell_{\text{place}}$. 
We analyze the different cases of the switching parameter $\gamma$.

\textbf{Case 1:} When $\gamma = 1$,
$\ell_{\text{lp}} = \ell_{\text{place}}$. 
The placing objective only guides the robot to reference points 
$\boldsymbol{p}_{\text{ref}} \in \{\boldsymbol{\hat{p}}^*_{\text{r}}, \boldsymbol{p}_{\text{near}}\}$ 
on the object, it does not obstruct the first-order improvement of the objective since $\boldsymbol{\hat{p}}^*_{\text{r}}$ and $\boldsymbol{p}_{\text{near}}$ is the optimal or an acceptable suboptimal solution returned by CSO for improving $\ell_{\text{obj}}$.

\textbf{Case 2:} When $\gamma = 0$, $\ell_{\text{lp}} = \ell_{\text{lift}}$ and the lifting phase moves the end-effector away from the object to an intermediate pose 
above the reference contact. Since CPO uses a short horizon, it is usually unable to reach $\mathcal{S}_{\text{contact}}$ while moving away from the object, and thus unable to obtain the gradient of $\ell_{\text{obj}}$. Therefore, $\ell_{\text{obj}}$ does not affect $\ell_{\text{lp}}$.

Since $\ell_{\text{lp}}$ is a convex combination of lift and place costs, $\ell_{\text{lp}}$ and $\ell_{\text{obj}}$ are non-conflicting objectives.